%% file: PaperForReview.tex
\documentclass[10pt,twocolumn,letterpaper]{article}

\usepackage[pagenumbers]{cvpr} 

\usepackage[utf8]{inputenc} 
\usepackage[T1]{fontenc}    
\usepackage[pdftex]{graphicx}
\usepackage{amsmath}
\usepackage{amssymb}
\usepackage{booktabs}
\usepackage{amsfonts}       
\usepackage{xcolor}         
\usepackage{wrapfig}
\usepackage{amsbsy}

\usepackage{multirow}
\usepackage{diagbox}
\usepackage{hhline}
\usepackage{boldline}
\usepackage{colortbl}
\usepackage{amsthm}
\usepackage{enumitem} 

\newcolumntype{?}{!{\vrule width 1pt}}
\newcommand{\grayc}{{\cellcolor[rgb]{0.945,0.945,0.945}}}

\newcommand*{\Scale}[2][4]{\scalebox{#1}{$#2$}}%

\input{math_commands.tex}

\usepackage[pagebackref,breaklinks,colorlinks]{hyperref}

\usepackage[capitalize]{cleveref}
\crefname{section}{Sec.}{Secs.}
\Crefname{section}{Section}{Sections}
\Crefname{table}{Table}{Tables}
\crefname{table}{Tab.}{Tabs.}
\Crefname{definition}{Def.}{Defs}
\Crefname{equation}{Eq.}{Eqs.}
\Crefname{proposition}{Prop.}{Props.}

\newtheorem{definition}{Definition}
\newtheorem{proposition}{Proposition}
\newtheorem{lemma}{Lemma}
\newtheorem{corollary}{Corollary}

\newcommand{\RQ}[1]{{\textbf{RQ#1}}}

\def\eg{\emph{e.g}.\xspace}
\def\ie{\emph{i.e}.\xspace}

\newcommand{\paragrapht}[1]{\noindent\textbf{#1}}

\usepackage{pifont}
\newcommand{\cmark}{\ding{51}}%
\newcommand{\xmark}{\ding{55}}%

\usepackage[accsupp]{axessibility}  

\def\cvprPaperID{8791} 
\def\confName{CVPR}
\def\confYear{2023}

\newcommand{\model}{\textcolor{black}{LBS}}
\newcommand{\geofeature}{\textcolor{black}{geometry-aware representation}}
\newcommand{\Geofeature}{\textcolor{black}{Geometry-aware representation}}

\begin{document}

\title{Learning Geometry-aware Representations by Sketching}

\author{Hyundo Lee, Inwoo Hwang, Hyunsung Go, Won-Seok Choi, Kibeom Kim, Byoung-Tak Zhang\\
AI Institute, Seoul National University\\
{\tt\small \{hdlee, iwhwang, hsgo, wchoi, kbkim, btzhang\}@bi.snu.ac.kr}
}
\maketitle

\begin{abstract}
   \input{sections/abstract.tex}
\end{abstract}

\section{Introduction}
\label{sec:intro}
\input{sections/intro.tex}

\section{Related work}
\label{sec:related}
\input{sections/related.tex}

\section{Mathematical framework}
\label{sec:property}
\input{sections/property.tex}

\section{Method}
\label{sec:method}
\input{sections/method.tex}

\section{Experiments}
\label{sec:exp}
\input{sections/experiments.tex}

\section{Conclusion}
\label{sec:conclusion}
\input{sections/conclusion.tex}

{\small
\bibliographystyle{ieee_fullname}
\bibliography{egbib}
}

\newpage
\onecolumn
\setcounter{section}{0}
\renewcommand\thesection{\Alph{section}}
\renewcommand\thesubsection{\thesection.\arabic{subsection}}
\section*{Appendix}
\input{sections/appendix}

\end{document}

%% file: math_commands.tex
\usepackage{amsmath,amsfonts,bm}

\def\eqref#1{equation~\ref{#1}}

\def\1{\bm{1}}

\def\vp{{\bm{p}}}

\DeclareMathAlphabet{\mathsfit}{\encodingdefault}{\sfdefault}{m}{sl}
\SetMathAlphabet{\mathsfit}{bold}{\encodingdefault}{\sfdefault}{bx}{n}

\def\gA{{\mathcal{A}}}

\def\gG{{\mathcal{G}}}

\def\gI{{\mathcal{I}}}

\def\gL{{\mathcal{L}}}

\def\gP{{\mathcal{P}}}

\def\gS{{\mathcal{S}}}

\def\gX{{\mathcal{X}}}

\def\gZ{{\mathcal{Z}}}

\def\sR{{\mathbb{R}}}

\usepackage{adjustbox}
\usepackage{lipsum}
\newcommand{\stdv}[1]{\scriptsize$\pm$#1}

\DeclareMathOperator*{\argmin}{arg\,min}

%% file: sections/abstract.tex
Understanding geometric concepts, such as distance and shape, is essential for understanding the real world and also for many vision tasks.
To incorporate such information into a visual representation of a scene, 
we propose learning to represent the scene by sketching, inspired by human behavior.
Our method, coined Learning by Sketching (LBS), learns to convert an image into a set of colored strokes that explicitly incorporate the geometric information of the scene in a single inference step without requiring a sketch dataset.
A sketch is then generated from the strokes where CLIP-based perceptual loss maintains a semantic similarity between the sketch and the image.
We show theoretically that sketching is equivariant with respect to arbitrary affine transformations and thus provably preserves geometric information.
Experimental results show that LBS substantially improves the performance of object attribute classification on the unlabeled CLEVR dataset, domain transfer between CLEVR and STL-10 datasets, and for diverse downstream tasks, confirming that LBS provides rich geometric information.

%% file: sections/intro.tex
Since geometric principles form the bedrock of our physical world, 
many real-world scenarios involve geometric concepts such as position, shape, distance, and orientation.
For example, grabbing an object requires estimating its shape and relative distance. 
Understanding geometric concepts is also essential for numerous vision tasks such as image segmentation, visual reasoning, and pose estimation~\cite{definition}.
Thus, it is crucial to learn a visual representation of the image that can preserve such information~\cite{geometry}, which we call \textit{geometry-aware representation}.

\begin{figure}[t]
\centering
\begin{minipage}{0.93\columnwidth}
\includegraphics[width=\columnwidth]{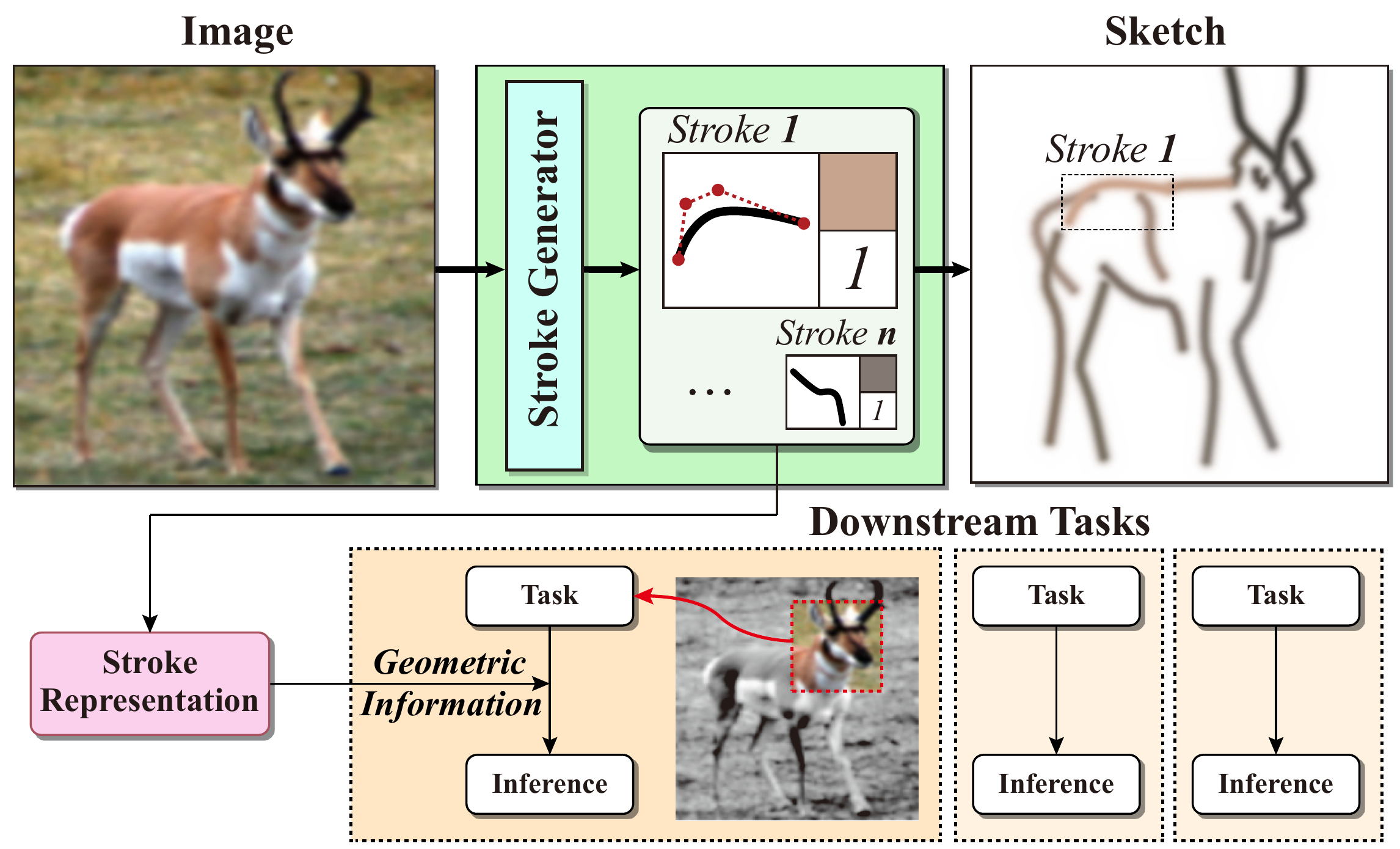}
\end{minipage}
\setlength{\belowcaptionskip}{-6pt}
\vspace{-0.1em}
\caption{
Overview of \model{}, which aims to generate sketches that accurately reflect the geometric information of an image.
A sketch consists of a set of strokes represented by a parameterized vector that specifies curve, color, and thickness.
We leverage it as a geometry-aware representation for various downstream tasks.
}
\label{fig:1}
\end{figure}

However, there is still a challenge in learning geometry-aware representations in a compact way that can be useful for various downstream tasks.
Previous approaches have focused on capturing geometric features of an image in a 2D grid structure, using methods such as handcrafted feature extraction~\cite{HoG, canny, lowlevelfeature1, lowlevelfeature2}, segmentation maps~\cite{u2net, grid}, or convolution features~\cite{lowlevelCNN, geoSSL}.
Although these methods are widely applicable to various domains, 
they often lack compactness based on a high-level understanding of the scene and tend to prioritize nuisance features such as the background.
Another line of work proposes architectures that guarantee to preserve geometric structure~\cite{equiCNN, steerable, e2equivariant, scaleCNN, groupStructure} or disentangle prominent factors in the data~\cite{betavae, anneal, betaTCVAE, factorvae}.
Although these methods can represent geometric concepts in a compact latent space, 
they are typically designed to learn features that are specific to a particular domain and often face challenges in generalizing to other domains~\cite{disentngleGeneralize}.

\begin{figure*}[t]
\centering
\begin{minipage}{1.0\textwidth}
\includegraphics[width=0.98\textwidth]{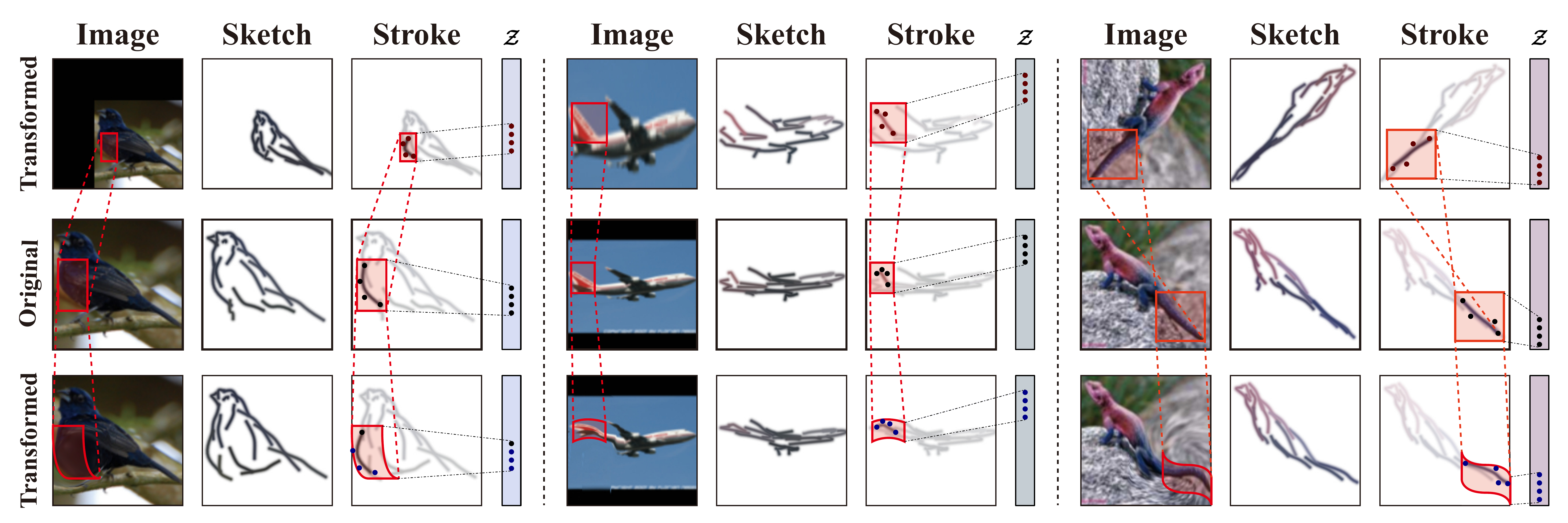}
\end{minipage}
\setlength{\belowcaptionskip}{-5pt}  
\vspace*{-1em}
\caption{Examples of how sketches capture essential geometric concepts such as shape, size, orientation, curvature, and local distortion.
The control points of each stroke compactly represent the local geometric information of the corresponding region in the sketch. The strokes as a whole maintain the geometric structure of the entire image under various transformations in the image domain.}
\label{fig:2}
\end{figure*}

In this study, we present a novel approach to learning geometry-aware representations via \textit{sketching}.
Sketching, which is the process of converting the salient features of an image into a set of color-coded strokes, as illustrated in ~\Cref{fig:1}, is the primary means by which humans represent images while preserving their geometry.
Our key idea is that sketches can be a compact, high-level representation of an image that accurately reflects geometric information.
Sketching requires a high-level understanding of the scene, as it aims to capture the most salient features of the image and abstract them into a limited number of strokes.
In addition, a sketch can be represented as a set of parameters by replacing each stroke with parametric curves.
Sketching has also been linked to how people learn geometric concepts~\cite{drawinghuman4}.
Based on these properties, we directly use strokes as a geometry-aware representation and utilize them for downstream tasks.
Under the theoretical framework of geometric deep learning~\cite{book}, we conduct theoretical analysis to validate the effectiveness of strokes as a representation and prove their ability to preserve affine transformations.

To validate our hypothesis,
we introduce Learning by Sketching (LBS), a method that generates abstract sketches coherent with the geometry of an input image.
Our model is distinct from existing sketch generation methods as it does not require a sketch dataset for training, which often has limited abstraction or fails to accurately reflect geometric information.
Instead, \model{} learns to convert an image into a set of colored Bézier curves that explicitly represent the geometric concepts of the input image.
To teach the style of sketching, we use perceptual loss based on the CLIP model~\cite{clip, clipasso}, which measures both semantic and geometric similarities between images and generated sketches.
We propose a progressive optimization process that predicts how strokes will be optimized from their initial positions through CLIP-based perceptual loss to generate abstract sketches in a single inference step.
As a result, \model~generates a representation that reflects visual understanding in a single inference step, without requiring a sketch dataset. This produces highly explainable representations through sketches, as illustrated in \Cref{fig:2}.

We conduct experimental analyses to evaluate the effectiveness of our approach, learning by sketching, by addressing multiple research questions in various downstream tasks, including:
\textbf{(i)} describing the relationships between geometric primitives,
\textbf{(ii)} demonstrating simple spatial reasoning ability by conveying global and local geometric information,
\textbf{(iii)} containing general geometric information shared across different domains, and
\textbf{(iv)} improving performance on FG-SBIR, a traditional sketch task.

%% file: sections/related.tex
\subsection{Geometry-aware representation learning}
Previous studies that aim to capture geometric information within the data can be categorized into two approaches: 
Firstly, explicitly capturing the geometric information with features on the image space, and secondly, inducing a compact feature vector that preserves the geometric structure of the data.

\paragrapht{Features on image space.}
The most straightforward approach to representing geometric structures in images is to establish features directly on the image space, \ie, to define features on the 2D grid structure.
Approaches that utilize this method include those that employ traditional edge detection algorithms~\cite{sobel, canny, hed, xdog}, low-level convolutional features~\cite{lowlevelCNN, clipasso}, and segmentation maps~\cite{segmentation, u-net, u2net}.
However, these features are unsuitable for conversion into compact vector representations for downstream tasks.
Alternative approaches include utilizing hand-crafted feature detection algorithms~\cite{HoG, orb,lowlevelfeature1, lowlevelfeature2}, clustering specific areas of an image into superpixels~\cite{superpixel, grid}, and representing the relationships between specific areas as a graph~\cite{graph1, graph2}.
Our approach is primarily notable for its ability to capture the most salient image features through an overall understanding of the scene, resulting in a concise representation.

\paragrapht{Geometric deep learning.}
Research has also been conducted to develop informative and compact representations that preserve geometric structures in a latent vector space.
One representative method is geometric deep learning~\cite{book}, which involves preserving the geometric properties of an image by learning an invariant (\ie, maintaining the output identically) or an equivariant representation (\ie, there exists an operator acting on the feature space corresponding to the transformations) with respect to geometric transformations.
Steerable networks~\cite{equiCNN, steerable, scaleCNN, e2equivariant} extend the translation equivariance of CNNs to a larger group of symmetries and introduce an efficient weight-sharing architecture.
However, designing their architecture requires complex analytic solutions, and they are constrained to particular transformations~\cite{implicitEquivariance}.
Another approach is to disentangle factors of variation into distinct latent subspaces~\cite{vae, betavae, anneal, factorvae, betaTCVAE, groupStructure, disentangledSSL}.
However, these methods face difficulties in handling changing factors of variation~\cite{disentngleGeneralize} and therefore struggle to generalize to arbitrary geometric transformations.
Sketching, in contrast to these methods, can naturally ensure equivariance with respect to geometric transformations by being defined on the image space.
In \cref{sec:property} we show that strokes can be equivariant with respect to arbitrary affine transformations.

\subsection{Sketch generation methods} \label{subsec:ref_sketch}
Human-drawn sketches are widely used as valuable descriptions in various applications, including image retrieval~\cite{retrieval1, retrieval2, retrieval3}, modeling 3D shapes and surfaces~\cite{3dmodeling1, 3dmodeling2, 3dmodeling3, 3dmodeling4}, and image editing~\cite{manipluation1, manipluation2, manipluation3, manipluation4}. 
Unfortunately, previous works that aim to learn sketching or drawing mostly endeavor to generate artistic works, rather than focusing on its applicability~\cite{sketchRnn, paintTransformer, spiral, spiralpp, rl2, optimization}.

Sketch generation models often rely on explicit sketch datasets to follow the style of human drawings~\cite{cyclesketch1, cyclesketch2, sketchrnn1, sketchrnn2, sketchformer, sketchBert}.
However, many existing datasets are not designed to preserve geometric information~\cite{sketchy, sketchyCOCO, quickdraw} or contain sketches without sophisticated image abstraction, resembling simple edge detection~\cite{data1, facesketch}.
This presents challenges for sketch models to accurately abstract an image while preserving its geometries.
Without sketch datasets, stroke-based rendering methods usually minimize the reconstruction loss which is related to the pixel-wise distance,
which promotes learning in the style of \textit{painting}~\cite{optimization, paintTransformer, renderer, spiral, spiralpp, rl2, strokenet}. 
There have also been studies that attempt to capture the geometric and semantic features of a scene through imitating the style of line drawing~\cite{linedrawings2, linedrawings3, linedrawings}.
However, these methods require numerous strokes to cover the whole canvas, primarily focusing on the superfluous background instead of salient features.

\paragrapht{Abstract sketches for representation.}
Only a few works aim to generate sketches for the sake of representing an image with a limited number of strokes.
Learning to Draw (LtD) \cite{learningToDraw} creates a representative sketch of an image and uses it as a communication channel between agents.
While it primarily focuses on communicating with a specific receiver,
we aim to leverage strokes as a general representation for various downstream tasks.
CLIPasso~\cite{clipasso} utilizes the capabilities of CLIP~\cite{clip}, a neural network trained on images paired with text,
to synthesize abstract sketches by maintaining both semantic and geometric attributes.
Generating representations with this method, however, is impractical due to its reliance on an optimization-based approach that is extremely time consuming and numerically unstable.
In contrast, our method generates a sketch of an image in a single inference step after training, 
and is relatively more stable and suitable as a representation.

%% file: sections/property.tex
In this section, we present our theoretical analysis on how sketches can facilitate geometry-aware representation.
We begin with a formal definition of sketch and strokes.
We then introduce a mathematical framework based on geometric deep learning, which aims to impose the geometric prior on the learned feature.\footnote{See \cite{book} for the extensive review.}
Finally, we show that with a loss function that satisfies certain conditions,
strokes can be equivariant with respect to arbitrary affine transformations and thus preserve geometric information.

\begin{figure*}[ht]
  \centering
  \includegraphics[width=0.94\textwidth]{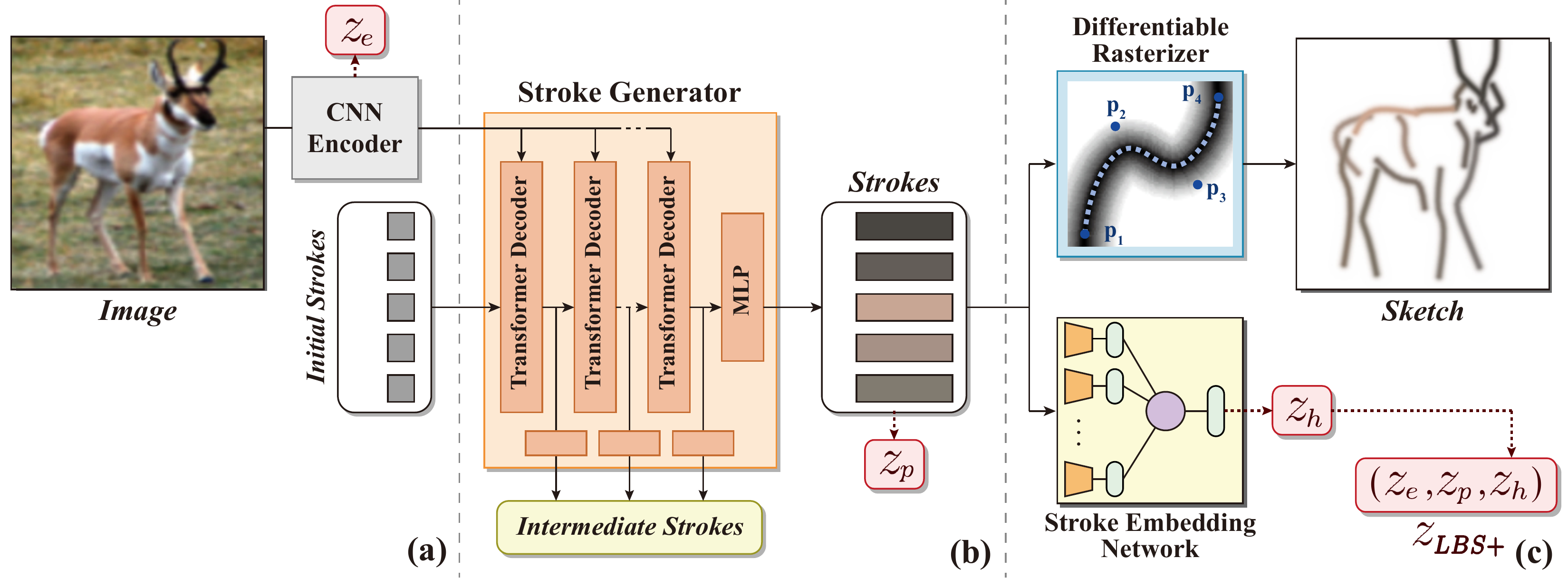}
  \setlength{\belowcaptionskip}{-8pt}  
  \vspace*{-0.5em}
  \caption{Visualization of the overall architecture of \model. \textbf{(a)} A CNN-based encoder extracts the feature of the image. \textbf{(b)} Initial strokes are updated repeatedly with a stroke generator based on a Transformer decoder. For the progressive optimization process, intermediate strokes are extracted from each intermediate layer. \textbf{(c)} The rasterized sketch and stroke embedding $z_h$ are generated from the strokes. Features derived from the CNN encoder, strokes, and stroke embedding are combined to form the final representation, $z_{LBS+}$.}
  \label{fig:3}
\end{figure*}

\subsection{Preliminaries and notations}
We consider an image $I\in\gI$ as \textit{signals} on a physical domain $\Omega$, which maps pixel coordinates to the $C$ color channels, \ie, $I: \Omega \rightarrow \sR^C$.
Although $\Omega$ is a 2D grid space for the real image,
we assume $\Omega$ as a homogeneous space, where each border is connected to the opposite side.

\paragrapht{Stroke and sketch.} 
A \textbf{\emph{stroke}} is a tuple $p=(t, c, w)$ composed of $k$ control points $t\in\Omega^k$, color $c\in[0, 1]^C$, and thickness $w\in[0, 1]$ of the Bézier curve.
A rendering process $r: \gP \rightarrow \gS(\subset \gI)$ converts a set of $n$ strokes $\vp=\left\{p^{(1)}, \cdots, p^{(n)}\right\} \in \gP$ into the corresponding image (see Appendix \ref{appendix:renderer} for further details),
which we denote as \textbf{\emph{sketch}} $S=r(\vp)\in\gS$.
Let $f: \gI\rightarrow\gP$ be the embedding function which maps an image to a set of strokes.
We denote $f(I)\in\gP$ and $[r\circ f](I)\in\gS$ respectively as a \textbf{\emph{stroke representation}} and a \textbf{\emph{sketch representation}} of the image $I$ w.r.t $f$.
To generate a sketch that accurately reflects the given image, we adopt a metric function $\gL: \gI\times\gS \rightarrow \sR$ that measures the distance between the image and the sketch.

\begin{definition}
\label{def:optimalrep}
    We denote $S_I=\argmin_{S\in\gS}{\gL(I,S)}$ as an \textbf{optimal sketch representation} of $I$.
\end{definition}

\paragrapht{\Geofeature{}.} 
We denote $\gG$ as a group of geometric transformations that acts on points in $\Omega$ (see Appendix~\ref{appendix:proof} for definitions).
For example, $\gG$ could be a group of translation or rotation.
$\rho$ is a group representation of $g\in\gG$ such that $\rho(g)$ is a linear operator which acts on $\gI$, \ie, $[\rho(g)I](u)=I(g^{-1}u)$ for $u\in\Omega$.

\begin{definition}
\label{def:geofeature}
$f(I)$ is a \textbf{\geofeature} of $I$ with respect to $\gG$
if there exists a group representation $\rho'$ such that $f(\rho(g)I)=\rho'(g)f(I)$ for $\forall I\in\gI$ and $\forall g\in\gG$.
\end{definition}

In this case, we say $f$ is a \textbf{$\gG$-equivariant map}, and the symmetry of $\gG$ is preserved by $f$.
For example, a convolution operator is a translation-equivariant map since the feature of the shifted image is equal to the shifted feature of the original image.
In other words, the convolution operator preserves the symmetry of the translation and produces a \geofeature{} with respect to it.

\subsection{Geometry-aware representations by sketching} \label{subsec:proposition}
We now show that stroke representation is a \geofeature{} with respect to arbitrary affine transformations.
Refer to Appendix \ref{appendix:proof} for detailed proofs.

\begin{proposition}
\label{prop:proposition}
If $\gL(\rho(g)I, \rho(g)S) = \lambda(g)\gL(I, S)$ where $\lambda: \gG \rightarrow \sR^+$ exists, and there exists a unique $f^\ast$ such that $[r\circ f^\ast](I) = S_I$ for $\forall I\in\gI$, then $f^\ast(I)$ is a \geofeature{} w.r.t. affine transformations $\gA$.
\end{proposition}

\cref{prop:proposition} states the conditions for $f$ to be an $\gA$-equivariant map.
Therefore, by designing stroke representations to be injective for $r$, and optimizing $f$ to minimize $\gL$ subject to the conditions outlined in \cref{prop:proposition},
$f$ generates a \geofeature{} with respect to arbitrary affine transformations.
In the next section, we propose an architecture and a loss function that satisfy these conditions.

%% file: sections/method.tex
We describe our model \model{}, which provides geometry-aware representations from strokes by learning to generate a sketch in a short inference time without using a sketch dataset.
We introduce an objective function and techniques for training \model{}.
We then propose a stroke embedding network that aggregates information of whole strokes.
Hereafter, we use the terms ``sketch'' and  ``image'' to denote realized sketch $S(\Omega)$ and realized image $I(\Omega)$ for simplicity.

\subsection{Learning by Sketching (LBS)} \label{subsec:lbs}
\model{} aims to generate a set of strokes to sketch the image and extract geometric information with color for downstream tasks.
To achieve the goal, we design a stroke generator by referring to Transformer-based sketch generation models~\cite{sketchformer, paintTransformer}.

\paragrapht{Architecture.}
As shown in \cref{fig:3}, \model~generates a set of strokes from trainable parameters, which are denoted as initial strokes.
Our stroke generator uses Transformer decoder layers and incorporates image features from a CNN encoder (ResNet18 is used in this paper).
Then, 2-layer MLP decoders decode a set of $n$ strokes from the output of the Transformer decoder.
Each stroke contains four control points that parameterize a cubic Bézier curve, as well as color and thickness information.
The sketch of an image is then created by rasterizing the generated set of strokes using a differentiable rasterizer~\cite{renderer}.
For a detailed description of our architecture, please refer to Appendix \ref{appendix:implementation}.

\paragrapht{Representation of \model.}
The primary output of \model{} is a set of strokes $\vp$, and a geometry-aware representation is obtained by using the flattened vector of the strokes $z_p=(p^{(1)}, ..., p^{(n)})$.
To complement information that may not be encoded with strokes such as texture,
we concatenate $z_p$ with the image features $z_e$ from the CNN encoder with global average pooling.
The resulting concatenated representation is denoted as $z_{LBS}$.
In addition, we can combine this with representation $z_h$ that aggregates the entire stroke, the result of which we denote as $z_{LBS+}$. Further details can be found in \cref{subsec:embed}.
The final representation space of \model~is as follows:
  \begin{equation}
    z_{LBS} = (z_e, z_p),\quad z_{LBS+} = (z_e, z_p, z_h).
  \end{equation}

\subsection{CLIP-based perceptual loss function}
As discussed in \cref{subsec:proposition}, proper $\gL$ must be given to train a geometry-aware representation of an image.
To this end, we adopt CLIP-based perceptual loss to measure the similarity between the sketch and the input image in terms of both geometric and semantic differences~\cite{clipasso}.
To generate a sketch that is \textit{semantically similar} to the input image, it should minimize the distance on the embedding space of the CLIP model $f_{C}$ as follows:
\begin{equation}
\label{eq:loss_semantic}
\gL_{semantic} = \sum_{A\in\gA} \phi (A(I), A(S)),
\end{equation}
where $\gA$ is the group of affine transformations and $\phi$ is the cosine similarity of the CLIP embeddings, \ie, $\phi(x, y)= \cos (f_C(x), f_C(y))$.
In addition, it compares low-level features of the sketch and the input image to estimate the \emph{geometric similarity} as follows:
\begin{equation}
\label{eq:loss_geometric}
\gL_{geometric} = \sum_{A\in\gA}\sum_{i} \phi_i(A(I), A(S)),
\end{equation}
where $\phi_i$ is the $\gL_2$ distance of the embeddings from the $i$-th intermediate layer of the CLIP model, \ie, $\phi_i(x, y) = \|f_{C, i}(x)-f_{C, i}(y)\|_2^2$.
Specifically, we used $i\in\{3, 4\}$ with the ResNet101 CLIP model, proposed as \cite{clipasso}.
The final perceptual loss is as follows: 
\begin{equation}
\label{eq:loss_perceptual}
\gL_{{percept}} = \gL_{geometric} + \lambda_s\cdot\gL_{semantic},
\end{equation}
with $\lambda_s=0.1$ as a hyperparameter.

\begin{figure}[t]
    \centering
    \includegraphics[width=0.99\columnwidth]{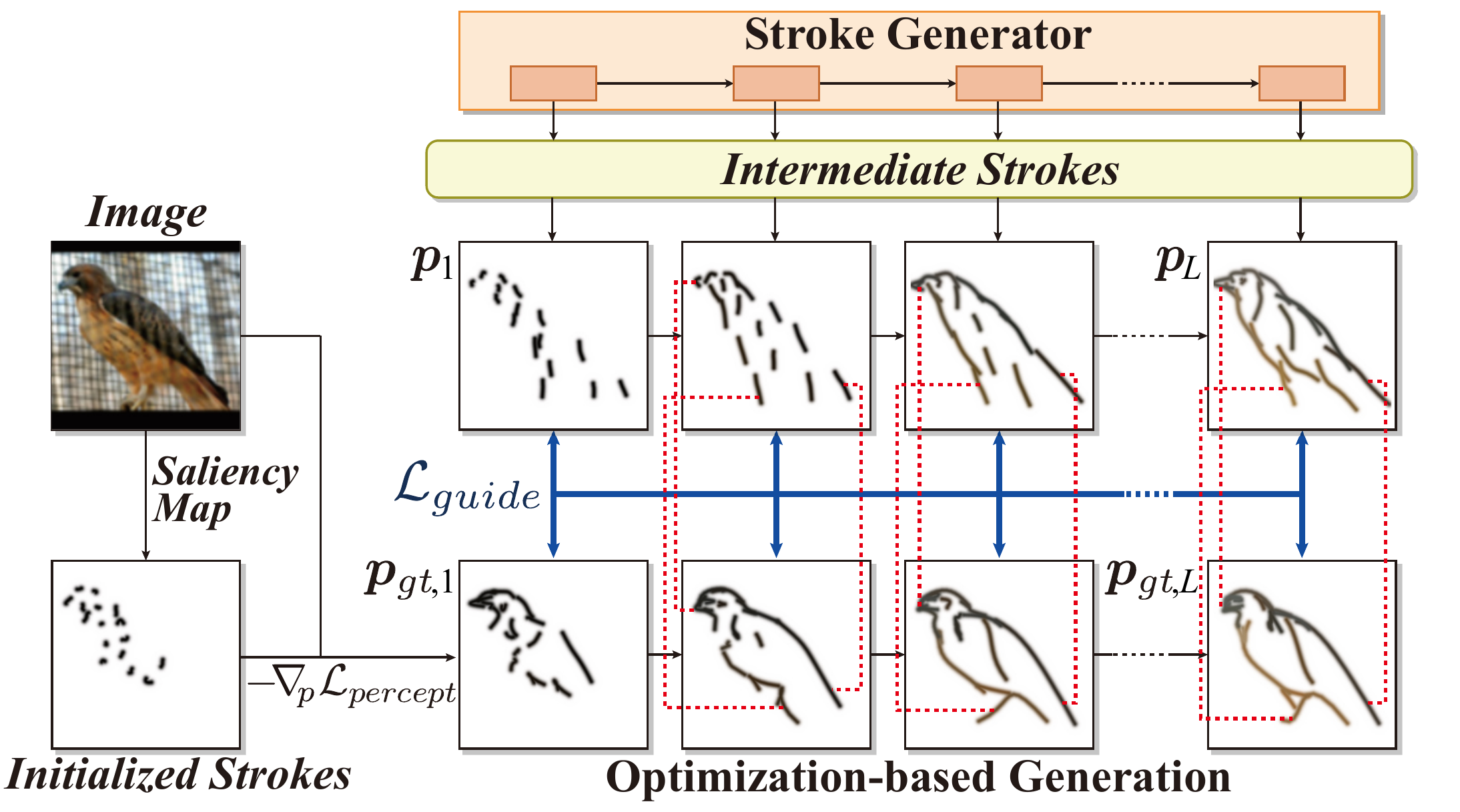}
    \setlength{\belowcaptionskip}{-8pt}  
    \vspace*{-1em}
    \caption{Visualization of the progressive optimization process. 
    Strokes initialized from the saliency map are optimized using the gradient of $\gL_{percept}$ until convergence and stored in a total of $L$ intermediate steps denoted by $\vp_{gt,l}$.
    Then, the stroke generator is trained through $\gL_{guide}$ to progressively predict each intermediate stroke $\vp_l$ from $\vp_{gt,l}$.
    }
    \label{fig:4}
\end{figure}

\subsection{Training with progressive optimization process} \label{subsec:progressive}
Abstracting an image into a sketch requires high-level reasoning,
and predicting the corresponding stroke parameters in a single inference step is a delicate process.
Although CLIPasso~\cite{clipasso} addresses these problems via an optimization-based approach,
its time-consuming properties ($\sim$2 min per image) make it impractical as a representation.
Moreover, as CLIPasso reported that the changes in initial strokes significantly impacted the final results,
directly optimizing each stroke through the optimizer for thousands of steps leads to highly variable results and numerical instability.

To generate a stable representation within a single inference step, we modify the optimization-based approach of \textit{CLIPasso}
and utilize it as desirable guidance for the stroke generator, which we call the guidance stroke and denote as $\vp_{gt}$.
To ensure high-quality sketch generation, each intermediate layer of the stroke generator is guided by strokes obtained at the corresponding intermediate step of the optimization process instead of relying solely on the final outcome.
This process, depicted in \cref{fig:4}, involves progressively updating strokes from the previous layer, which can be modeled by residual blocks from the Transformer architecture. 
By decomposing the prediction of the entire optimization process into predicting changes in relatively short intervals for each intermediate layer, our approach leads to better-quality sketches.

However, the process through which CLIPasso initializes strokes via stochastic sampling from the saliency map results in unstable outcomes.
To address this issue, we select initialized strokes deterministically by selecting the most salient regions from the saliency map in a greedy manner.
The color of the stroke is determined during the initialization step by referencing the corresponding region in the original image.
More details are described in Appendix \ref{appendix:optim}.

To ensure that the guidance loss is invariant to the permutation of the order of each stroke,
we use the Hungarian algorithm~\cite{hungarian} which is widely used for the assignment problem. 
The guidance loss is as follows:
\begin{align}
  \gL_{guide}=\sum_{l=1}^L \min_{\sigma} \sum_{i=1}^n \gL_1(p_{gt,l}^{(i)}, p_l^{(\sigma(i))}),
\end{align}
where $\sigma$ is a permutation of the stroke index, and $p^{(i)}_l$ is the $i$-th decoded stroke from the $l$-th intermediate layer with $L$ layers in total. $\gL_1$ is the L1 distance, and $p^{(i)}_{gt,l}$ is the guidance stroke corresponding to $p^{(\sigma(i))}_l$.

\subsection{Stroke embedding network} \label{subsec:embed}
As strokes primarily express localized and fine-grained details, they can be less suitable for representing coarse-grained information.
To address this, we offer an additional representation that combines full strokes through a stroke embedding network.
Stroke embedding network $h:\gP\rightarrow\gZ_{h}$ maps the set of $n$ strokes into embedding space $\gZ_{h}$.
Since the rasterized sketch is invariant to the stroke indices (except for contact point between strokes),
we followed the permutation invariant model from DeepSets~\cite{deepSets} to generate stroke embedding.
The choice of the loss function $\gL_{embed}$ to train $z_h\in\gZ_h$ is unconstrained (\eg, Cross-Entropy or InfoNCE loss in our paper). 
During the training phase, the gradient from $h$ is also passed to $f$ to capture coarse-grained information.

The final objective is as follows:
\begin{equation} \label{eq:loss_total}
  \gL_{LBS} = \gL_{{percept}} + \lambda_g \cdot \gL_{guide} + \lambda_e \cdot \gL_{embed},
\end{equation}
where $\lambda_g$ and $\lambda_e$ are hyperparameters (for specific values, refer to Appendix \ref{appendix:implementation}).

%% file: sections/experiments.tex
\subsection{Research questions and task designs}
We investigate the effectiveness and generalizability of our approach by designing multiple research questions, as well as corresponding downstream tasks for each question.
Detailed information about each experimental setup is described in Appendix \ref{appendix:setup}.

\vspace*{0.2em}
\paragrapht{RQ1. Effects of equivariance:}
\textit{What is the benefit of equivariant representation, and does LBS achieve it?}

From the definition of equivariance, we expect that an equivariant model with respect to geometric transformation would have generalization capability for augmented images with transformation. 
To verify this, we conduct experiments on the Rotated MNIST (rotMNIST) dataset~\cite{rotmnist}, which consists of rotated images from the MNIST dataset.
We rotate images by 30$^\circ$ and 45$^\circ$ for training and classify the digit of images rotated by 0$^\circ$ and 90$^\circ$ to evaluate whether \model{} generalizes the information from the training domain to the test domain. 
To verify that our model satisfies equivariance outside the training domain, we classify the rotation and reflection between an original image and the same image rotated with an interval of 90$^\circ$ and horizontally flipped.

\vspace*{0.2em}
\paragrapht{RQ2. Effects on geometric primitives and concepts:}
\textit{How well does \model{} represent geometric primitives and understand their conceptual relationships?}

To evaluate whether \model{} is suitable for representing geometric primitives and concepts, we perform a classification task on the Geoclidean dataset~\cite{geoclidean}.
The Geoclidean dataset consists of realized images from a concept of Euclidean geometry, \eg, black parallel lines on a white background.
Geoclidean is divided into two categories: Geoclidean-Elements and Geoclidean-Constraints, visualized in \cref{figure:geoclidean} on Appendix \ref{appendix:setup}.
By providing a limited training set consisting of only 10 images per concept, we evaluate whether \model{} can effectively learn the high-level relationships between each primitive and generalize them across different examples by classifying the concept of the test image.

\begin{table}[t]
    \centering
    \caption{Quantitative results of \model~and the baselines evaluated on the corresponding tasks on the rotMNIST and Geoclidean.
    }
    \begin{adjustbox}{width=\linewidth}
    \begin{tabular}{clcc|cc}
    \toprule
    \multirow{2}{*}{\begin{tabular}[c]{@{}c@{}}La-\\bel\end{tabular}} & \multirow{2}{*}{Method} & \multicolumn{2}{c}{\begin{tabular}[c]{@{}l@{}}rotMNIST (30$^\circ$, 45$^\circ$)\end{tabular}} & \multicolumn{2}{c}{Geoclidean} \\ \cmidrule{3-6} 
                           & & 0$^\circ$, 90$^\circ$                    & Rotation                   & Constraints               & Elements                  \\ 
    \midrule
    \multirow{5}{*}{\cmark}
    & CE                                      &  62.72\stdv 1.89                  & 73.14\stdv 0.35                   & 53.89\stdv 1.58                   & 70.57\stdv 4.29                   \\
    & SupCon \cite{supcon}                    &  \grayc 65.25\stdv 1.05           & \grayc 73.37\stdv 0.20            &  \grayc 42.41\stdv 3.16           & \grayc 55.83\stdv 4.28            \\ 
    & LtD-diff  \cite{learningToDraw}         &  60.09\stdv 1.16                  & 28.91\stdv 2.15                   & 57.26\stdv 2.19                   & 69.47\stdv 2.11                   \\
    & \text{E(2)}-CNN~\cite{e2equivariant}   &  \grayc \textbf{85.12}\stdv 0.17  & \grayc \textbf{77.00}\stdv 0.39   & \grayc \textbf{71.03}\stdv 1.94   & \grayc 69.28\stdv 1.46            \\
    & \textbf{\model~(CE)}                                &  65.00\stdv 3.32                  & 75.81\stdv 0.30                   & 50.01\stdv 1.58                   & \textbf{81.06}\stdv 3.14          \\ \midrule
    \multirow{5}{*}{\xmark} & SimCLR \cite{simclr}                    & \grayc \textbf{47.66}\stdv 0.41   & \grayc \textbf{72.96}\stdv 0.84   & \grayc 32.04\stdv 0.64            & \grayc 65.14\stdv 4.11           \\
    & $\beta$-TCVAE \cite{betaTCVAE}          & 37.84\stdv 1.06                   & 60.14\stdv 0.32                   & 17.18\stdv 1.35                   & 33.82\stdv 1.64                   \\
    & GeoSSL \cite{geoSSL}                    & \grayc  -                         & \grayc -                          & \grayc 18.66\stdv 3.33            & \grayc 33.47\stdv2.80           \\
    & HoG \cite{HoG}                          & 25.42                             & 64.94                             & 23.82                             & 52.05                             \\
    & \textbf{\model}                                    & \grayc 47.31\stdv 2.53            & \grayc 70.02\stdv 0.64            & \grayc \textbf{47.43}\stdv 1.34   & \grayc \textbf{81.34}\stdv 0.16   \\ 
    \bottomrule
    \end{tabular}
    \end{adjustbox}
    \label{table:1}
    \vspace*{-1em}
\end{table}

\begin{table*}[t]
    \centering
    \caption{Quantitative results on CLEVR dataset for labeled and unlabeled settings and for CLIP and HoG features. RC, LC, BC refer to inferring the color of the rightmost, leftmost, bottommost object, and Size, Shape, Material refer to inferring the size, shape, material of the rightmost object, respectively.
    Third refers to inferring the color of the third object from the right,
    and Shift refers to predicting the color of the rightmost object after shifting the initial rightmost object to the left by 0.15 times the image width.
    }
    \begin{adjustbox}{width=0.9\textwidth}
    \begin{tabular}{clcccccccc}
    \toprule
    Label                 & Method            & RC    & LC    & BC    & Size  & Shape & Material & Third & Shift \\ \midrule  
    \multirow{6}{*}{\cmark}
    & CE                                        & 98.71\stdv{0.10}          & 76.72\stdv{0.86}          & 76.02\stdv{0.90}          & 92.51\stdv{0.40}          & 49.97\stdv{0.29}           & 64.38\stdv{1.90}             & 40.66\stdv{0.20}          & \textbf{62.06}\stdv{1.62} \\
    & SupCon  \cite{supcon}                     & \textbf{98.75}\stdv{0.08} & 63.82\stdv{2.36}          & 66.04\stdv{2.65}          & 91.88\stdv{0.37}          & 49.15\stdv{0.85}           & 59.20\stdv{0.57}             & 37.93\stdv{0.27}          & 56.05\stdv{2.56}          \\
    & LtD-diff  \cite{learningToDraw}           & 62.29\stdv{0.48}          & 14.01\stdv{0.42}          & 15.84\stdv{0.43}          & 63.98\stdv{3.38}          & 43.96\stdv{3.05}           & 54.47\stdv{0.45}             & 16.47\stdv{0.59}          & 17.21\stdv{0.29}          \\
    & E(2)-CNN~\cite{e2equivariant}        & 98.50\stdv{0.10}          & 70.10\stdv{1.15}          & 73.51\stdv{2.50}          & 89.84\stdv{0.46}          & 45.85\stdv{0.93}           & 63.05\stdv{0.59}             & \textbf{41.95}\stdv{0.18} & 59.29\stdv{0.91}          \\
    & \textbf{\model~(CE)}                                 & 97.49\stdv{0.22}          & \textbf{81.79}\stdv{0.27} & \textbf{84.09}\stdv{0.84}          & \textbf{93.22}\stdv{0.29} & \textbf{70.03}\stdv{0.68}  & \textbf{86.84}\stdv{0.54}    & 38.23\stdv{0.25}          & 51.56\stdv{0.16}          \\
    \midrule
    
    \multirow{6}{*}{\xmark}
    & SimCLR \cite{simclr}                      & 60.61\stdv{1.24}          & 63.89\stdv{2.35}          & 63.77\stdv{2.29}          & 83.35\stdv{0.60}          & 41.95\stdv{0.33}          & 51.66\stdv{0.24}          & 33.42\stdv{0.55}          & 43.05\stdv{0.55}          \\
    & E(2)-CNN~\cite{e2equivariant}    & 53.50\stdv{7.30}          & 54.60\stdv{7.34}          & 55.52\stdv{7.60}          & 83.52\stdv{1.56}          & 42.06\stdv{1.12}          & 53.04\stdv{2.81}          & 30.74\stdv{2.84}          & 38.03\stdv{4.44}          \\
    & $\beta$-TCVAE  \cite{betaTCVAE}           & 17.09\stdv{0.20}          & 17.50\stdv{0.33}          & 20.04\stdv{0.71}          & 71.27\stdv{0.10}          & 36.30\stdv{0.10}          & 56.11\stdv{0.13}          & 15.38\stdv{0.19}          & 16.35\stdv{0.18}          \\
    & GeoSSL \cite{geoSSL}                      & 20.16\stdv{0.63}          & 20.54\stdv{0.87}          & 21.61\stdv{0.67}          & 73.79\stdv{0.78}          & 44.08\stdv{1.10}          & 54.54\stdv{2.94}          & 15.39\stdv{0.16}          & 16.94\stdv{0.34}          \\
    & DefGrid \cite{grid}                      & 73.81\stdv{0.91}          & 73.96\stdv{1.02}          & 73.38\stdv{0.80}          & 81.50\stdv{0.22}          & 46.34\stdv{0.77}          & 68.65\stdv{1.17}          & 24.90\stdv{0.27}          & 36.28\stdv{0.13}          \\
    & \textbf{\model}                                      & \textbf{84.31}\stdv{0.08} & \textbf{80.47}\stdv{0.78} & \textbf{83.00}\stdv{0.39}          & \textbf{92.66}\stdv{0.41}          & \textbf{70.01}\stdv{0.53}          & \textbf{85.52}\stdv{0.43}          & \textbf{37.41}\stdv{0.29}          & \textbf{49.32}\stdv{0.17} \\
    \midrule

    \multirow{2}{*}{}
    & CLIP \cite{clip}                          & 37.39 & 39.1  & 54.98 & 77.51 & 66.91 & 72.66    & 34.75 & 34.80       \\
    & HoG \cite{HoG}                            & 56.83 & 49.69 & 58.69 & 81.73 & 61.14 & 68.38    & 24.28 & 33.25       \\
\bottomrule
\end{tabular}
\end{adjustbox}
\label{table:2}
\end{table*}

\begin{table*}[t]
    \centering
    \caption{Quantitative results of models trained on STL-10 dataset and evaluated on CLEVR dataset.}
    \begin{adjustbox}{width=0.9\textwidth}
    \begin{tabular}{clcccccccc}
    \toprule
    Dataset                 & Method            & RC    & LC    & BC    & Size  & Shape & Material & Third & Shift \\ \midrule
\multirow{6}{*}{\begin{tabular}[c]{@{}c@{}}STL-10\\ $\downarrow$\\ CLEVR\end{tabular}}
& LtD \cite{learningToDraw}                 & 18.70\stdv{0.02}          & 19.15\stdv{0.77}          & 21.69\stdv{1.05}          & 73.13\stdv{3.21}          & 47.88\stdv{1.13}          & 59.11\stdv{2.34}          & 16.47\stdv{0.59}          & 17.21\stdv{0.29}          \\
& SimCLR \cite{simclr}                      & 26.97\stdv{0.28}          & 26.26\stdv{0.88}          & 38.88\stdv{0.19}          & 82.33\stdv{0.64}          & 56.88\stdv{1.23}          & 77.39\stdv{0.35}          & 23.45\stdv{0.35}          & 24.95\stdv{0.17}          \\
& E(2)-CNN~\cite{e2equivariant}    & 26.82\stdv{2.23}          & 25.40\stdv{1.20}          & 39.30\stdv{2.80}          & 79.48\stdv{0.71}          & 53.54\stdv{0.84}          & 72.18\stdv{0.34}          & 23.35\stdv{1.01}          & 24.80\stdv{1.08}          \\
& $\beta$-TCVAE  \cite{betaTCVAE}           & 17.04\stdv{0.76}          & 18.46\stdv{2.36}          & 18.88\stdv{1.14}          & 71.32\stdv{0.67}          & 37.14\stdv{0.12}          & 56.81\stdv{0.26}          & 15.63\stdv{1.35}          & 16.43\stdv{0.38}          \\
& GeoSSL \cite{geoSSL}                      & 14.41\stdv{0.55}          & 14.73\stdv{0.11}          & 14.37\stdv{0.66}          & 77.47\stdv{5.17}          & 48.10\stdv{1.90}          & 55.63\stdv{2.86}          & 13.71\stdv{0.45}          & 13.63\stdv{0.13}          \\
& DefGrid \cite{grid}                      & \textbf{66.74}\stdv{1.10}          & \textbf{68.60}\stdv{0.08}          & \textbf{68.35}\stdv{0.82}          & 78.97\stdv{0.29}          & 40.92\stdv{0.44}          & 62.98\stdv{1.31}          & 23.48\stdv{0.49}          & 33.07\stdv{0.52}          \\
& \textbf{\model}                                      & 53.87\stdv{2.72}          & 53.85\stdv{2.20}          & 61.84\stdv{2.23}          & \textbf{92.46}\stdv{0.62} & \textbf{59.52}\stdv{0.45}          & \textbf{80.06}\stdv{1.29}          & \textbf{30.75}\stdv{0.41}          & \textbf{39.95}\stdv{1.15}          \\
\bottomrule
\end{tabular}
\end{adjustbox}
\vspace*{-0.5em}
\label{table:3}
\end{table*}

\begin{table}[t]
    \centering
    \caption{Quantitative results of models trained on CLEVR dataset and evaluated on STL-10 dataset classification.}
    \begin{adjustbox}{width=0.99\linewidth}
    \begin{tabular}{clc|lc}
    \toprule
    & \multicolumn{2}{c|}{\begin{tabular}[c]{@{}l@{}}Labeled\end{tabular}} & \multicolumn{2}{c}{\begin{tabular}[c]{@{}l@{}}Unlabeled\end{tabular}} \\
    Dataset                 & Method            & Accuracy & Method & Accuracy \\ \midrule
    \multirow{6}{*}{\begin{tabular}[c]{@{}c@{}}CLEVR\\ $\downarrow$\\ STL-10\end{tabular}}
    & CE                                        & 46.15\stdv{0.12} & SimCLR \cite{simclr}                      & 41.68\stdv{0.05} \\
    & SupCon  \cite{supcon}                     & 43.41\stdv{0.36}\grayc & $\beta$-TCVAE  \cite{betaTCVAE}           & 27.35\stdv{0.38}\grayc \\
    & LtD-diff  \cite{learningToDraw}           & 50.81\stdv{0.67} & GeoSSL \cite{geoSSL}                      & 35.93\stdv{0.96} \\
    & E(2)-CNN~\cite{e2equivariant}        & 45.19\stdv{0.84}\grayc & E(2)-CNN~\cite{e2equivariant}    & 38.50\stdv{0.49}\grayc \\
    & \textbf{\model~(CE)}                                  & \textbf{56.48}\stdv{0.89}  & DefGrid \cite{grid}                      & 33.13\stdv{0.17} \\
    & & & \textbf{\model}                                      & \textbf{55.35}\stdv{0.18}\grayc \\
    \bottomrule
\end{tabular}
\end{adjustbox}
\label{table:4}
\vspace*{-1em}
\end{table}

\vspace*{0.2em}
\paragrapht{RQ3. Local geometric information and spatial reasoning:}
\textit{How effectively does \model{} reflect local geometric information and extend it to spatial reasoning tasks?}

We validate whether \model{} can reflect the local geometric information of each object in a synthetic photo dataset consisting of multiple objects, using the CLEVR dataset~\cite{clevr}.
We train our model with very limited descriptions, where the label for the entire scene is provided as the rightmost object or without any descriptions at all. 
We validate the effectiveness of our representation by evaluating its ability to successfully classify attributes that are not provided as labels, such as determining the color of the leftmost object.
Additionally, we test its ability to perform simple spatial reasoning, such as shifting the rightmost object and inferring the attribute of the current rightmost object.

\vspace*{0.2em}
\paragrapht{RQ4. Domain transfer:}
\textit{Can the geometric concepts of \model{} trained in a specific domain be extended to other domains?}

To investigate whether the learned representation 
within a specific image domain provides meaningful geometric information across other domains,
we evaluate the model by shifting the distribution from the STL-10~\cite{stl10} dataset to CLEVR and vice versa.
Evaluation in the STL-10 dataset is conducted through object classification.

\vspace*{0.2em}
\paragrapht{RQ5. Traditional sketch tasks:}
\textit{Does geometry-aware representation from \model{} help in traditional FG-SBIR tasks?}

We assess the impact of geometry-aware representations acquired by \model{} on a conventional sketch task by evaluating them with the fine-grained sketch-based image retrieval (FG-SBIR) task on the QMUL-ShoeV2 Dataset~\cite{retrieval1}.
The aim of FG-SBIR is to match a given sketch of a shoe to its corresponding photo image.
Since geometric information is essential for distinguishing each shoe, 
we investigate whether leveraging representations from \model{} can improve performance on the FG-SBIR task.

\subsection{Experimental setup}

\paragrapht{Implementation details.}
As samples from rotMNIST and Geoclidean can be directly represented with a few strokes,
we replace $\gL_{percept}$ and $\gL_{guide}$ with $\gL_1$ loss, which does not rely on a pre-trained CLIP model.
For the STL-10 dataset, providing mask information for the background with U2-Net~\cite{u2net} improves the quality of generated sketches.
For $\gL_{embed}$, we use Cross-Entropy loss for supervised training which we denote as \model~(CE) in \cref{table:1}, \ref{table:2} and \ref{table:4}.
For the unlabeled setting, the model is trained without $\gL_{embed}$ and is denoted \model~with no parentheses.
For evaluation, we use $z_{LBS+}$ for \model~(CE), and $z_{LBS}$ for \model{}.
All evaluations are performed with the common protocol of linear classification.
For more detailed information on our implementation, please refer to Appendix \ref{appendix:implementation}.

\paragrapht{Baselines.}
We compare our models with:
\vspace{-0.5em}
\begin{itemize}
\setlength\itemsep{-0.4em}
    \item Models which only use image encoder:
    trained with Cross-Entropy, SimCLR~\cite{simclr}, or SupCon loss~\cite{supcon}.
    \item Equivariance models: $\text{E(2)}$-Equivariant Steerable CNNs ($\text{E(2)}$-CNN)~\cite{e2equivariant}, and work by Novotny \etal, denoted as GeoSSL~\cite{geoSSL}.
    \item Disentanglement model: 
    $\beta$-TCVAE~\cite{betaTCVAE}.
    \item Sketch model for communication between agents: Learning to Draw (LtD)~\cite{learningToDraw} 
    \item Grid-based superpixel model: DefGrid\cite{grid}
    \item Handcrafted feature descriptor: HoG~\cite{HoG}.
    \item Pre-trained CLIP model used for training \model{}.
\end{itemize}

\subsection{Results and analysis}

\paragrapht{Analysis on RQ1 and RQ2:}
\cref{table:1} shows the experimental results for \RQ{1} and \RQ{2}, which can be summarized as: 
\textbf{(a)} Providing a strong inductive bias for equivariance, as shown in the case of $\text{E(2)}$-CNN, leads to good performance on the tasks of \RQ{1}. 
While \model{} outperformed LtD-diff, which also utilizes a sketch-based approach, it requires additional validation and refinement.
\textbf{(b)} For \RQ{2}, \model{} performs substantially better than the other baselines except for the supervised setting on Geoclidean-Constraints, showing that it successfully captures geometric concepts. 
In particular, Geoclidean-Elements perform better when trained without labels.
\textbf{(c)} Disentanglement methods and handcrafted features offer less informative features for both tasks.

\vspace{0.2em}
\paragrapht{Analysis on RQ3:}
\cref{table:2} summarizes the experiments conducted on the CLEVR dataset, demonstrating that \model{} provides valuable information for tasks of \RQ{3}.
The key observations from \cref{table:2} are:
\textbf{(a)} While the baselines in the supervised setting struggle to generalize on unlearned attributes, \model~(CE) can provide more general information.
\textbf{(b)} The results of \model{} show a larger performance gap in the unlabeled setting, with its performance being almost as good as \model~(CE).
\textbf{(c)} Embeddings from pre-trained CLIP perform less effectively than unlabeled baselines, such as SimCLR, on tasks related to local geometry. 
Also, handcrafted features yielded limited effectiveness.
\textbf{(d)} Methods that focus on preserving global transformations and disentanglement are not effective in environments of \RQ{3}.

\vspace{0.2em}
\paragrapht{Analysis on RQ4:}
\cref{table:3,table:4} summarize the results for shifting domains from STL-10 to CLEVR, and vice versa.
As shown in \cref{figure:Transfer} on Appendix \ref{appendix:qualitative}, \model{} trained on STL-10 can successfully describe scenes from CLEVR, while \model{} trained on CLEVR can abstractly depict natural images from STL-10, leading to substantially high classification accuracy.
These results indicate that learning geometric information, such as position and curvature, by sketching salient parts of a scene successfully induces learning of general features across domains.
Furthermore, LtD-diff, which also utilizes a sketch-based approach, demonstrated good performance in transferring CLEVR to STL-10.

\vspace{0.2em}
\paragrapht{Analysis on RQ5:}
\cref{table:fgsbir} presents the results for the FG-SBIR task. 
Compared to the baseline trained without sketching, \model{} can accurately predict the corresponding shoe during the test phase.
The huge drop in performance upon removing strokes and stroke embedding from $z_{LBS+}$ indicates that the geometry-aware representation provided by \model{} is also valuable for traditional sketch tasks.

\subsection{Comparisons with stroke-based methods} \label{subsec:style}
\cref{table:5} compares \model~to other stroke-based generation methods on STL-10 dataset, and \cref{figure:sketh_method} visualizes the results of each method. 
The evaluation is based solely on strokes, and \model only uses $z_p$ as a representation.
\model~($\gL_{percept}$) achieved superior results compared to \model~($\gL_1$) and the Paint Transformer~\cite{paintTransformer} based on $\gL_1$ loss.
We also compare our results to CLIPasso modified to use our deterministic initialization process, suggesting that strokes generated from CLIPasso may not be suitable as a representation.
LtD-diff, which generates sketches to solve a communication game, often omits important geometric and color information, resulting in low performance.

\begin{table}[t]
    \centering
    \caption{Top-1 and Top-10 accuracy on QMUL-Shoe-v2 dataset, using Triplet Loss for $\gL_{embed}$ in \model{}.
    }
    \vspace*{-0.2em}
    \begin{adjustbox}{width=0.7\linewidth}
    \begin{tabular}{lccc}
        \toprule
        Method & Rep & Top-1 & Top-10 \\ \midrule
            \multirow{2}{*}{\model~(Triplet)} & $z_{LBS+}$ & 40.8\stdv{4.19} & 88.3\stdv{1.76} \\
            & $z_e$ & 30.8\stdv{3.89}\grayc & 77.8\stdv{3.89}\grayc \\
            Triplet & $z_e$ & 32.5\stdv{1.73} & 80.0\stdv{4.36} \\
        \bottomrule
    \end{tabular}
    \end{adjustbox}
    \label{table:fgsbir}
    \vspace*{-0.5em}
\end{table}

\begin{table}[t]
    \centering
    \caption{Comparison between stroke-based generation methods.}
    \vspace*{-0.3em}
    \begin{adjustbox}{width=0.9\linewidth}
    \begin{tabular}{lccccc}
        \toprule
                                            & RC                & BC                & Size              & Shape             & Material          \\ \midrule
        LtD-diff  \cite{learningToDraw}     & 18.70             & 21.69             & 73.13             & 47.88             & 59.11             \\
        Paint~\cite{paintTransformer}       & 51.96\grayc             & 60.40\grayc              & 77.35\grayc             & 44.30\grayc              & 70.02\grayc             \\
        CLIPasso~\cite{clipasso}          & 12.55           &  13.55           & 55.35                  & 35.40    & 50.00          \\
        \model ~($\gL_1$)                            & 34.43\grayc             & 57.81\grayc             & 78.32\grayc             & 40.99\grayc             & 61.59\grayc             \\
        \textbf{\model ~($\gL_{percept}$)}                       & \textbf{55.15}    & \textbf{60.63}    & \textbf{90.14}    & \textbf{51.23}    & \textbf{76.16}    \\
        \bottomrule
        \end{tabular}
        \end{adjustbox}
    \label{table:5}
    \vspace*{-0.2em}
\end{table}

\begin{figure}[t]
    \centering
    \includegraphics[width=0.8\columnwidth]{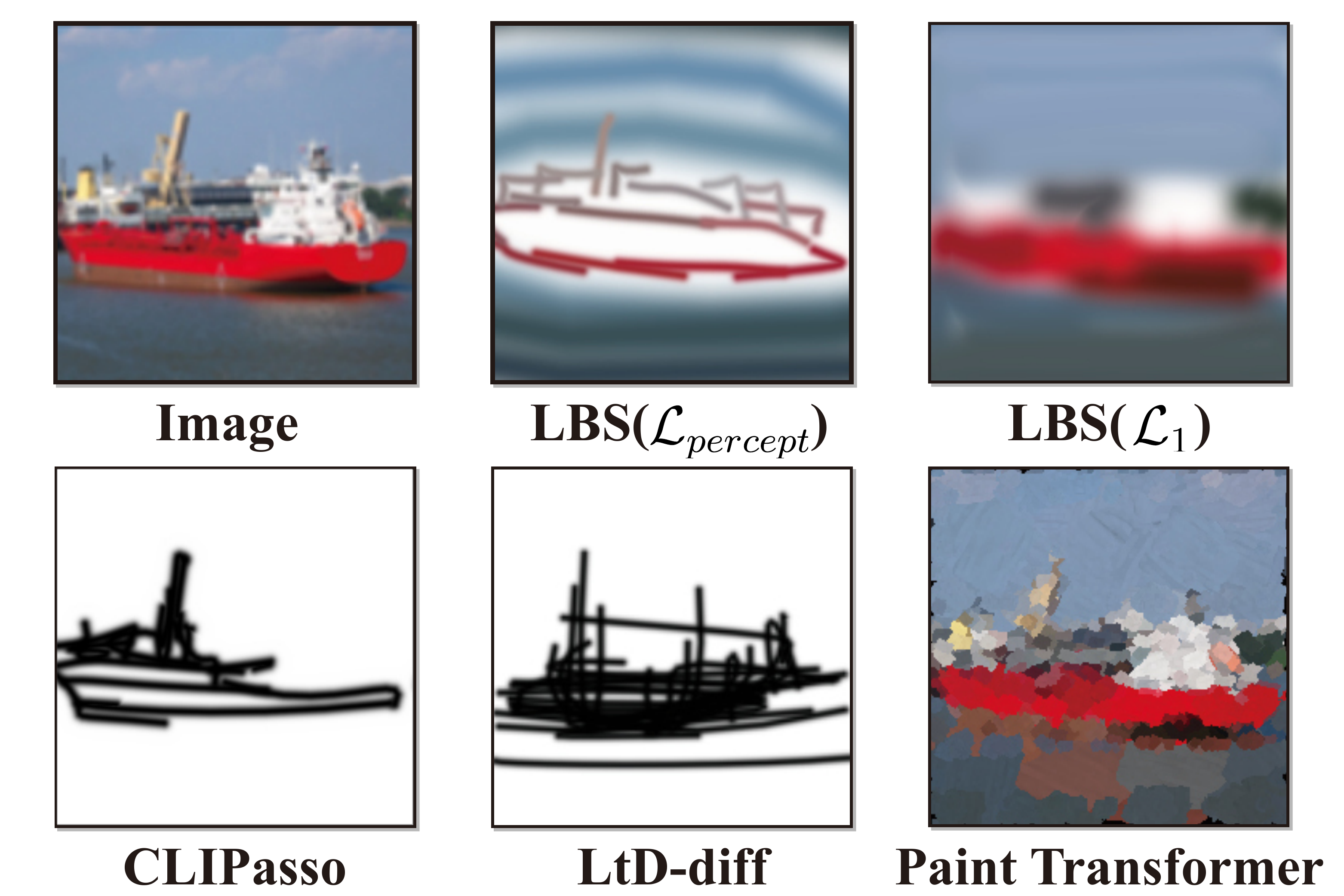}
    \vspace*{-0.5em}
    \caption{Comparisons to other stroke-based methods. \model~encodes both the geometric and semantic information.}	
    \label{figure:sketh_method}
    \vspace*{-1em}
\end{figure}

%% file: sections/conclusion.tex
We propose a novel representation learning perspective through sketching, a visual expression that retains geometric information.
By considering the mathematical concept of equivariance, we provide a formal definition of sketch and stroke, and prove that strokes can preserve the information of arbitrary affine transformations.
Our sketch generation model, \model{}, minimizes CLIP-based perceptual loss and supports the idea that sketches can effectively convey geometric information. 
The experiments shed light on the potential benefits of leveraging the process of sketching to represent high-level geometric information explicitly.
These benefits include learning equivariant representations, understanding geometric concepts, improving spatial reasoning abilities, and acquiring general geometric information across different domains.

Still, there are several limitations to our method. 
The training process of \model{} relies on pre-trained CLIP and optionally U2-Net, and requires generating guidance strokes through a time-consuming optimization-based process.
Moreover, while our work has demonstrated the theoretical strengths of sketching in terms of equivariance, a clearer methodology and further experimental analysis are needed to fully achieve this potential. 
Future work should address these issues and extend the methodology, focusing on scalability and including more complex and realistic tasks.

\vspace{-1em}
\paragraph{Acknowledgement}
This work was partly supported by the Korean government (2021-0-02068-AIHub/25\%, 2021-0-01343-GSAI/25\%, 2022-0-00951-LBA/25\%, 2022-0-00953-PICA/25\%).

%% file: sections/appendix.tex
\section{Details on optimization-based guidance stroke generation process} \label{appendix:optim}

In this section, we describe the details of the optimization method for generating $\vp_{gt}$ described in \Cref{subsec:progressive}, 
by referring to the pipeline of CLIPDraw~\cite{clipdraw} and CLIPasso~\cite{clipasso} with modifications.

\paragraph{Saliency map.}
We use the saliency map of CLIPasso to sample the initial strokes. 
We first employed U2-Net~\cite{u2net} to remove the background from the STL-10 dataset.
Then, we obtained the relevancy map by averaging the attention heads of each self-attention layer of the ViT-32/B CLIP model.
Finally, we generated the saliency map by multiplying the edge map of the masked image using edges from the XDoG edge detection algorithm~\cite{xdog}.

\paragraph{Initialization process}
As described in \Cref{subsec:progressive}, CLIPasso stochastically samples the initialized strokes, which results in the guidance $\vp_{gt}$ being highly unstable. 
Instead, we propose to deterministically sample the initialized stroke with low computational cost using the following process. 
First, when initializing the first stroke, we sample the coordinate with the highest weight of the saliency map as the first control point of the first stroke.
The remaining control points are sampled by adding a small random perturbation. 
Next, we lower the weight of the saliency map by a value related to the distance between the sampled point. Specifically, we subtract the weight of the saliency map proportionately by the exponential of the negative squared distances scaled by $\sigma$,
i.e., $\text{saliency}(u) \leftarrow \text{saliency}(u) -I(u)\cdot\exp(-(u-t^{(1)}_1)/\sigma^2)$ where $\sigma=5$, $u\in\Omega$, and $t^{(1)}_1$ denotes the first control point of the first stroke.
Finally, we sample the next stroke from the modified saliency map, and the process is repeated for the number of strokes used by our model.

We observed that optimizing the stroke color with $\gL_{percept}$ led to the generated results approaching the style of painting with colors appearing similar to the object's texture, which is not the desired as described in \cref{subsec:style}.
To tackle this problem, we sample the color of each stroke during the initialization process and use it as the final color, instead of optimizing its value.
To determine the color of each stroke, we use the color of the image corresponding to the initial position of the first control point of each stroke.
Then, we refined the color with a small adjustment,
$\tilde{c} = \frac{\sigma((2c-1)\beta) - \sigma(-\beta)}{\sigma(\beta) - \sigma(-\beta)}$ where $\tilde{c}$ is the adjusted color ranged from 0 to 1.
$\sigma(\cdot)$ represents the logistic Sigmoid function, and $\beta$ is a hyperparameter that is set to 5 to promote color contrast.

\begin{figure*}[h!]
    \centering
    \includegraphics[width=0.95\textwidth]{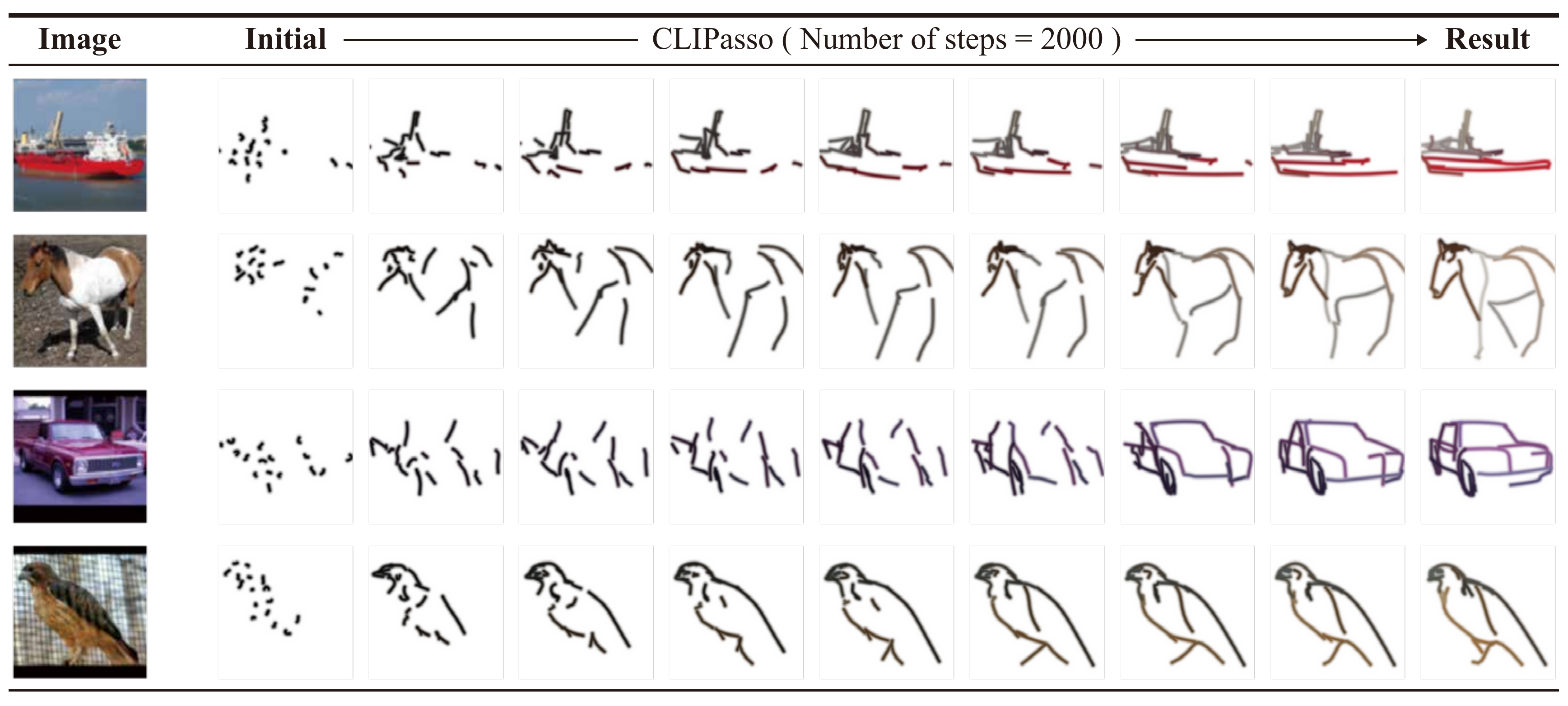}
    \caption{Visualization of optimization-based generation process. From left to right: examples of a given image, initialized stroke, and optimization results at 50, 100, 200, 400, 700, 1000, 1500, and 2000 iterations, respectively.}
    \label{figure:clipasso_examples}
\end{figure*}

\paragraph{Hyperparameters on stroke optimization}

We use the same optimization scheme as CLIPasso with our initialization process.
We repeat the optimization process for 2000 iterations and evaluate $\gL_{percept}$ with augmented images with random affine transformations. 
The strokes are trained with an Adam optimizer using a learning rate of 1.0.
We save 8 intermediate results, namely at steps 50, 100, 200, 400, 700, 1000, 1500, and 2000, for $\vp_{gt,l}$ where $l\in\{1, 2, ..., 8\}$, as shown in \Cref{figure:clipasso_examples}.

\section{Additional experimental setup} \label{appendix:setup}

\paragraph{Geoclidean dataset}

\begin{figure}[h!]
    \centering
    \includegraphics[width=0.95\columnwidth]{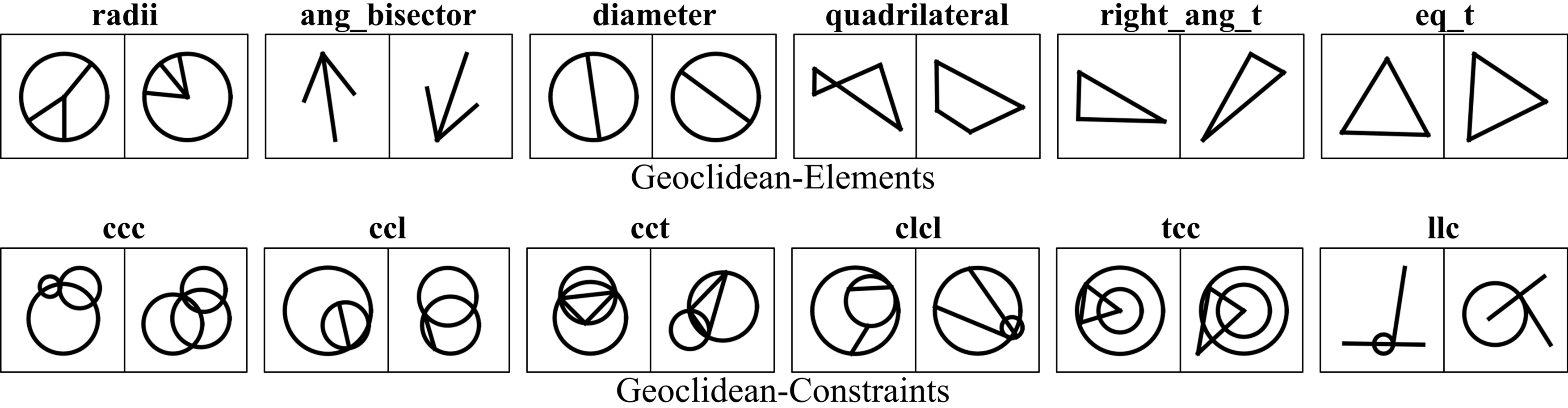}
    \vspace*{-0.2em}
    \caption{Example images of the Geoclidean dataset.}
    \label{figure:geoclidean}
    \vspace*{-0.2em}
\end{figure}

The Geoclidean dataset is divided into two categories: Geoclidean-elements and Geoclidean-concepts. 
The concepts of Geoclidean-Elements and Geoclidean-Constraints are derived from the first book of Euclid's Elements (\eg, equilateral triangle, rhomboid, or parallel lines) and from the relationship between primitive elements (\eg, 3 circles with the same point of intersection), respectively.
17 concepts consist Geoclidean-Elements dataset,
while 20 concepts consist Geoclidean-Constraints dataset.
Since the task on Geoclidean is designed to generalize each geometric concept with a small amount of data, we use only 10 examples per class for training.
The training is performed for 5000 epochs, and we evaluate the model with the weights obtained in the final epoch. 
To measure classification accuracy, we use 500 test images per class. 
The images are rescaled to have a resolution of 64$\times$64, and we apply random augmentations such as rotation in the range of (-90$^\circ$, 90$^\circ$), translations with a factor of (-0.1, 0.1), and scaling with a factor of (0.8, 1.2).

\paragraph{Rotated MNIST dataset}

The training is performed for 200 epochs with 2000 training samples and 500 validation samples per rotation.
To predict the transformation between unaltered and transformed images, we specifically classify the out-of-distribution transformation, namely 0, 90, 180, and 270 degrees rotation with or without horizontal reflection, resulting in 8 possible transformations.
For each model trained on rotMNIST, we concatenate the two latent vectors obtained with the input images for linear probing.
We rescale the image to have a 32$\times$32 resolution and invert the pixel values to make the image with black foreground on a white background.
For random augmentation, we use random cropping with a scale of (0.7, 1.0).

 \paragraph{CLEVR dataset}

\begin{figure}[h!]
    \centering
    \vspace*{-0.2em}
    \includegraphics[width=0.5\columnwidth]{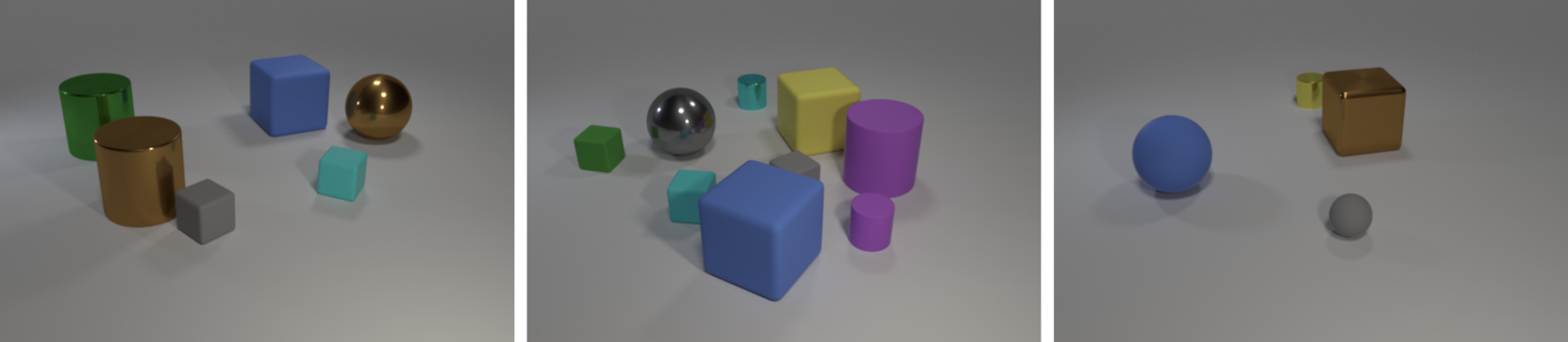}
    \vspace*{-0.2em}
    \caption{Example images of CLEVR dataset.}
    \label{figure:clevr}
    \vspace*{-0.2em}
\end{figure}

The CLEVR dataset contains rendered images of 3 to 10 objects. Each object has a set of attributes that include 8 colors, 3 shapes, 2 sizes, 2 materials, and 2D pixel positions, with question-answer pairs that evaluate a high-level understanding of the scene.
On the CLEVR dataset, we trained the model for 200 epochs using 8000 training samples and 2000 validation samples, with the images resized to 128$\times$128 resolution. 
Our task on CLEVR dataset explicitly requires the local geometric information of each object.
To preserve such local geometric information, we chose small adjustments for random augmentation, such as random cropping with a scale of (0.9, 1.0) with weak color jittering, as large adjustments can harm these attributes and negatively affect our model (e.g., random cropping can remove the rightmost object from the scene where the label is given as the rightmost object in the original scene).

\paragraph{STL-10 dataset}
On STL-10, we train the model for 50 epochs using 90k unlabelled images for training and 10k for validation. Then, we select the model with the lowest validation loss and evaluate the performance of tasks on the CLEVR dataset. Since cropping small regions can make the scene unclear to sketch, we use random cropping with a scale (0.8, 1.0), horizontal flipping, and color jittering for random augmentation.

\paragraph{QMUL-ShoeV2 dataset}
The QMUL-ShoeV2 dataset consists of 2,000 shoe photos and 6,648 corresponding human-drawn shoe sketches.
We use triplet loss as $\gL_{embed}$, with $z_{LBS+}$ generated from a shoe image and a corresponding sketch as input a positive pair, and $z_{LBS+}$ for another shoe image as a negative pair.
We trained the model for 1000 epochs using 200 shoes as the test split, with the images resized to 128$\times$128 resolution. 
Since the sketches provided by the dataset are sketches with only black strokes, LBS also generates sketches with the stroke color fixed to black. 
We use random augmentation, which is the same setting as STL-10. 

\section{Implementation details \& hyperparameter settings.} \label{appendix:implementation}

For the implementation of \model{}, we selected the cubic Bézier curve for parameterizing each stroke.
As the samples from rotMNIST and Geoclidean can be accurately reconstructed with a few strokes, 
we replace $\gL_{percept}$ and $\gL_{guide}$ with $\gL_1$ loss, which does not rely on a pre-trained CLIP model.
To ensure strokes do not deviate from the grid space and are standardized to be from left to the right, we penalize $\vp$ as follows:

\begin{align}
    \gL_{boundary} &= \sum_{p\in\vp} \left[ \sum_{t\in p} \max(0, -1+||t||_{\infty})\right], \\
    \gL_{align} &= \sum_{p\in\vp} \left[ \max(0, t_{1, x} - t_{4, x}) \right], \\
    \gL_{penalty} &= \gL_{boundary} + \gL_{align},
\end{align}

where $t_{1, x}$, $t_{4, x}$ is the $x$ coordinate of the first and last control point of $p$, respectively.
The final loss is $\gL_1 + \lambda_p\cdot\gL_{penalty} + \lambda_e\cdot\gL_{embed}$, 
where $\lambda_p=0.1, \lambda_e=0.01$ on Geoclidean and $\lambda_p=1, \lambda_e=0.1$ on rotMNIST.

We selected a 4-layer CNNs and a 2-layer Transformer decoder with a hidden size of 256 for the CNN encoder and the stroke generator.
The embedding network is based on the Invariant model from DeepSets~\cite{deepSets}, and we used 2-layer MLPs for embedding each stroke and a 1-layer MLP for the aggregated embedding. 
We use 4 strokes for the Geoclidean dataset and 8 for the rotMMNIST dataset. 
The thickness is fixed for the Geoclidean dataset, and the color is fixed to be black for both experiments.
We train the models using AdamW optimizer~\cite{adamw} with a learning rate of 1e-3.

On CLEVR, STL-10, and QMUL-ShoeV2 datasets, the CNN encoder is based on ResNet18 architecture, and the stroke generator is based on an 8-layer Transformer decoder with 512 hidden units. 
We use an embedding network with the same architecture as the one used in the Geoclidean experiment. 
To ensure consistency in stroke order, we arrange $p_{gt}$ to follow a left-to-right order. 
We use 24 strokes for the CLEVR dataset, 20 for the STL-10 dataset, and 10 for the QMUL-ShoeV2 dataset, with thickness fixed for all datasets. 
The hyperparameters for the loss function are set to $\lambda_g=10, \lambda_e=0.1$ for CLEVR dataset, $\lambda_g=1, \lambda_e=0.1$ for STL-10 dataset, and $\lambda_g=10, \lambda_e=1$ for QMUL-ShoeV2 dataset. 
In STL-10, we assign 4 strokes as background strokes to express information about the background and trained by minimizing the $\gL_1$ loss with the masked foreground.
We use an AdamW optimizer with a learning rate of 2e-4.
We use the ResNet101 architecture of the publicly available pre-trained CLIP model to measure $\gL_{percept}$. The effects of changing the CLIP model architecture are described in Appendix \ref{appendix:architecture}.

Training \model{} with $\gL_{LBS}$ requires a large amount of GPU memory. Therefore, to ensure memory efficiency during training of \model{}~(InfoNCE) in Appendix \ref{appendix:loss_ablation} and \ref{appendix:rep_ablation}, we employed a momentum encoder and queue from MoCo~\cite{moco}. Specifically, we used a batch size of 32.

\paragraph{Evaluation protocol.}
All evaluations are conducted with linear probing, where we train a single linear head for evaluation.
We select the model with the best validation results during training for evaluation. 
We train the linear head 2 times, each time for 100 epochs using the SGD optimizer with a learning rate of 1 and 0.1, respectively.
We evaluate the accuracy with test samples for all epochs and learning rates and report the best value as the evaluation accuracy.

\paragraph{Baselines.}
For $\text{E(2)}$-CNN, we use the $C_8$-equivariant CNN model for rotMNIST and Geoclidean dataset, the $C_8C_4C_1$-equivariant WideResNet 16-8 model~\cite{wrn} for CLEVR dataset, and the $D_8D_4D_1$ model for STL-10 dataset~\cite{steerable}.
E(2)-CNN is trained with Cross-Entropy loss in the labeled setting and InfoNCE loss in the unlabeled setting. 
For LtD~\cite{learningToDraw}, we select the Object Oriented game-different setup for rotMNIST, Geoclidean, and CLEVR datasets (LtD-diff), while selecting the original setup for STL-10 (LtD-original).
For $\beta$-TCVAE~\cite{betaTCVAE}, we adopt the hyperparameters from the public repository\footnote{https://github.com/YannDubs/disentangling-vae}, with a latent size of 16 for rotMNIST and Geoclidean datasets, and 32 for CLEVR and STL-10.
For DefGrid~\cite{grid}, we use a 15$\times$15 grid for the CLEVR dataset and a 20$\times$20 grid for the STL-10 dataset.
Since our experiments use a lower resolution than the image resolution used in DefGrid, we increase $\lambda_{area}$ to 0.02 and 0.1 for the CLEVR and STL-10 datasets, respectively. 
We conducted experiments using our implementation of the work by Novotny \etal~\cite{geoSSL} (referred to as GeoSSL in our paper), as no publicly available code was available.

\section{Differentiable rasterizer} \label{appendix:renderer}
The rendering process $r$ maps the set of strokes $\vp\in\gP$ generated by the stroke generator of \model{} to a signal $S$ on the physical domain $\Omega$. 
We calculate the CLIP-based perceptual loss between the realized sketch $S(\Omega)$ and the original image $I(\Omega)$ by computing every pixel value in $\Omega$ from the signal $S=r(I)$. 
To enable the gradient of the loss to propagate back to the stroke generator, we use a differentiable rasterizer \cite{renderer} that allows the mapping of $\vp$ to $S(\Omega)$ to be differentiable. 
\cite{renderer} formulated the differentiable pixel rasterization process by setting the value of each pixel to be an exponential of the negative squared distance between the pixel and the vector primitive $p=(t_1, t_2, t_3, t_4, c, w)$ (i.e., stroke) as follows:

\begin{align} \label{eq:renderer_1}
    r'(u;p)=\exp(-d_{cur}^2(u,p)/w^2) \cdot c,
\end{align}

where $d_{cur}$ is the Euclidean distance between the parametric curve $C(s,p)$ ($0\leq s\leq 1$) represented by $p$ and each pixel coordinate $u\in\Omega$.

\begin{align}
\begin{split}
    d_{cur}^2(u;p)=&\min_s{||C(s,p)-u||^2_2} \\
    &\text{s.t.} \quad 0\leq s\leq 1
\end{split}
\end{align}

To better represent the curvature information of the image, we chose a cubic Bézier Curve consisting of four control points among various parametric curves.
The first and last control points are set as the starting and ending points of the curve, respectively.

\begin{align}
    C(s, p)=(1-s)^3t_1 + s(1-s)^2t_2 + s^2(1-s)t_3 + s^3t_4
\end{align}

To enable more pixels to be affected by strokes during the early stages of training without changing the stroke thickness while generating a sharp sketch at the later stages, we anneal the exponential value applied to Euclidean distance and $w$ in \cref{eq:renderer_1}. 
Specifically, we increase the exponential value to increase linearly from 1 to 2 during the stages of training and decrease the value of $w$ linearly from $2w$ to $w$.

We use the \textit{over} composition \cite{renderer} to extend this process to multiple strokes. 
The over composition of two images $A$ and $B$ is defined as $c_{over}(A, B) = A + B(1-A)$, and can be recursively defined for any sequence of images. 
For numerical stability, the cumulative product expressed as an exponentiated sum of log differences is used as follows:

\begin{align}
    r(\Omega;\vp) = \sum_{i=1}^k r'(\Omega;p^{(i)}) \odot \exp\left(\sum_{j=1}^{i-1}\log\left(1-r'(\Omega;p^{(j)})\right)\right)
\end{align}

However, the over composition has a problem of color blending for colored strokes (e.g., yellow color is rasterized at the intersection of red and green strokes). 
To ensure that the color of the stroke painted later replaces the color of the previous stroke, we modify the cumulative product as follows.

\begin{align}
    r(\Omega;\vp) = \sum_{i=1}^k r'(\Omega;p^{(i)}) \odot \exp\left(\sum_{j=1}^{i-1}\log\left(1-r'_{intensity}(\Omega;p^{(j)})\right)\right),
\end{align}

where $r'_{intensity}$ refers to the value without multiplying the color parameter, i.e., $r'_{intensity}(u;p)=\exp(-d_{cur}^2(u,p)/w^2)$.

\section{Omitted proofs} \label{appendix:proof}

We first start with the definition of groups, group actions, and group representations.

\begin{definition}
    \vspace{0.75em}
    A \textbf{group} $\gG$ is a set with a composition $\circ:\gG\times\gG\rightarrow\gG$ (where we denote as $g\circ h=gh$) satisfying:
    \begin{enumerate}[label=(\roman*)]
        \item \textbf{Associativity}: $(gh)i=g(hi), \quad\forall g,h,i\in\gG$.
        \item \textbf{Identity}: $\forall g\in\gG$, there exists a unique $e\in\gG$ s.t. $eg=ge=g$. 
        \item \textbf{Inverse}: $\forall g\in\gG$, there exists a unique $g^{-1}\in\gG$ s.t. $gg^{-1}=g^{-1}g=e$. 
        \item \textbf{Closure}: $\forall g,h\in\gG$, we have $gh\in\gG$.
    \end{enumerate}
\end{definition}

\begin{definition}
    \vspace{0.75em}
    For $\gG$ with identity element $e$ and a set $\gX$,
    a (left) \textbf{group action} of $\gG$ on $\gX$ is a function $(g, x)\mapsto g.x$ that satisfies:
    \begin{enumerate}[label=(\roman*)]
        \item \textbf{Identity}: $e.x=x, \quad\forall x\in\gX$.
        \item \textbf{Compatibility}: $g.(h.x)=(gh).x, \quad\forall g,h\in\gG, \forall x\in\gX$.
    \end{enumerate}
\end{definition}

where we defined the group of geometric transformations $\gG$ acting on $\gI$ as $(g.I)(u)=I(g^{-1}u)$ in Section \ref{sec:property}.

\begin{definition}
    \vspace{0.75em}
    A \textbf{group representation} $\rho:\gG\rightarrow GL(V)$ is a map that assigns each $g$ to an invertible matrix $\rho(g)$, 
    such that $\rho(gh)=\rho(g)\rho(h)$ for $\forall g,h\in\gG$.
\end{definition}

Since actions on signals are linear, i.e., $g.(\alpha I+\beta I')(u)=(\alpha I+\beta I')(g^{-1}u)=\alpha(I)(g^{-1}u)+\beta(I')(g^{-1}u)=\alpha(g.I)(u)+\beta(g.I')(u)$ for any scalars $\alpha, \beta$ and signals $I, I'\in\gI$,
we can describe group action of $g$ acting on $\gI$ as a linear matrix $\rho(g)\in GL(|\Omega|, \sR)$ which acts on $\Omega$. Since $\gS\subset \gI$, sketch shares the same physical domain $\Omega$ as $\gI$.
Thus, finding the linear operator $\rho'(g)$ for equivariance is trivial; we can let $\rho'(g)$ as the same operator which acts in the image space, $\rho'=\rho$.
For example, equivariance with respect to rotation could be achieved if the sketch representation of a rotated image is equally rotated and for other arbitrary geometric transformations.

We first show that sketch representation which minimizes a metric function with some constraints, is a geometry-aware representation w.r.t arbitrary geometric transformations.
Then, we show that stroke representation is also a geometry-aware representation w.r.t affine transformations.

\begin{lemma} \label{lemma:sketch}
    If $\gL$ holds the following conditions for $\forall I\in\gI, \forall S\in\gS, \forall g\in\gG$:
    \begin{enumerate}[label=(\roman*)]
        \item $\gL(\rho(g)I, \rho(g)S) = \lambda(g)\gL(I, S)$ where $\lambda: \gG \rightarrow \sR^+$.
        \item $\exists! S_I$ s.t. $S_I = \argmin_{S\in\gS}\gL(I,S)$ 
        \item $\rho(g)S\in\gS$.
    \end{enumerate}
    Then optimal sketch representation $S_I$ with given $\gL$ is a \geofeature{} with respect to any $\gG$.
\end{lemma}

\begin{proof}
    With $\rho(g)S\in\gS$, we can also let $\rho$ as a group representation of $\gS$. \\
    Let $S'_I=\argmin_{S\in\gS}\gL(I,\rho(g)^{-1}S)$. \\
    \begin{equation}
        \gL(I, S_I)=\min_{S\in\gS}\gL(I,S)=\min_{S\in\gS}\gL(I,\rho(g)^{-1}S)=\gL(I, \rho(g)^{-1}S'_I),
    \end{equation}
    \begin{equation} \label{eq:lemma1}
        \therefore \; \rho(g)^{-1}S'_I=S_I,\;\argmin_{S\in\gS}\gL(I, \rho(g)^{-1}S)=\rho(g)\argmin_{S\in\gS}\gL(I, S).
    \end{equation}

    \begin{align}
    \begin{split}
        S_{\rho(g)I} &= \argmin_{S\in\gS}\gL(\rho(g)I, S) \\
                     &= \argmin_{S\in\gS}\gL(\rho(g)I, \rho(g)\rho(g)^{-1}S) \\ 
                     &= \argmin_{S\in\gS}\lambda(g)\gL(I, \rho(g)^{-1}S) \quad \Scale[0.7]{\leftarrow condition. (i)} \\
                     &= \argmin_{S\in\gS}\gL(I,\rho(g)^{-1}S) \\ 
                     &= \rho(g)\argmin_{S\in\gS}\gL(I,S) \quad \Scale[0.7]{\leftarrow equation. (\ref{eq:lemma1})} \\
                     &= \rho(g)S_I
    \end{split}
    \end{align}

    Therefore, $S_I$ is a \geofeature{} w.r.t any geometric transformations, where $\rho'=\rho$.
\end{proof}

We now define optimal sketch function $[r\circ f]^*(I)=S_I$ which minimizes $\gL$ in \cref{coroll:optimal_fn},
and show that $\gL_{percept}$ in \cref{subsec:lbs} can satisfy the first condition of \cref{lemma:sketch} w.r.t affine transformations in \cref{coroll:augmentation}.

\begin{corollary} \label{coroll:optimal_fn}
    If \cref{lemma:sketch} holds,
    a function $[r\circ f]$ that minimizes the metric between $I$ and $[r\circ f](I)$, \\
    i.e., $[r\circ f]^*=\argmin_{f}\gL(I, [r\circ f](I))$, generates a optimal sketch representation $[r\circ f]^*(I)=S_I$.
\end{corollary}

\begin{proof}
    $\gL(I, [r\circ f]^*(I))=\min_{f}\gL(I, [r\circ f](I))=\min_{S\in\gS}\gL(I, S)=\gL(I, S_I)$. \\
    Since optimal sketch representation is unique, $[r\circ f]^*(I)=S_I$.
\end{proof}

\begin{corollary} \label{coroll:augmentation}
    \vspace{0.75em}
    For any $\gL':\gI\times\gS\rightarrow\sR$, 
    $\gL(I, S)=\mathbb{E}_{g}[\gL'(\rho(g)I, \rho(g)S)]$ satisfies the first condition of \cref{lemma:sketch}.
\end{corollary}

\begin{proof}
    Since $\gG$ is closed for its composition, 
    \begin{align}
    \begin{split}
        \gL(\rho(g)I, \rho(g)S)&=\mathbb{E}_{g'}[\gL'(\rho(g')\rho(g)I, \rho(g')\rho(g)S)] \\
                               &=\mathbb{E}_{g'}[\gL'(\rho(g'g)I, \rho(g'g)S)] \\
                               &=\mathbb{E}_{g''}[\gL'(\rho(g'')I, \rho(g'')S)] \\
                               &=\gL(I, S).
    \end{split}
    \end{align} 
\end{proof}

\cref{lemma:sketch} implies that optimal $[r\circ f]$ which minimizes the metric function $\gL$ that satisfies the conditions in \cref{lemma:sketch} can make sketch $S=[r\circ f](I)$ as a \geofeature{} w.r.t arbitrary geometric transformations.
We now show that a set of strokes $\vp=f(I)$ can also be a \geofeature{} w.r.t arbitrary affine transformations $\gA$.

\begin{lemma} \label{lemma:stroke}
    \vspace{0.75em}
    If the following condition holds:
    \begin{enumerate}[label=(\roman*)]
        \item $S_I$ is a \geofeature{} w.r.t $\gA$.
        \item The rendering process $r:\gP\rightarrow\gS$ is a $\gA$-equivariant map.
        \item There exists a function $f^*:\gI\rightarrow\vp$ s.t. $[r\circ f^*]=[r\circ f]^*$ and $r(f^*(I))=r(f^*(I'))\Rightarrow f^*(I)=f^*(I')$.
    \end{enumerate}
    Then $f^*(I)$ with given $\gL$ is a \geofeature{} with respect to $\gA$.
\end{lemma}

\begin{proof}
    Since $r$ is a $\gA$-equivariant map, there exists a $\rho'$ s.t. $\rho(a)r(\vp)=r(\rho'(a)\vp)$ for $\forall a\in\gA, \vp\in\gP$. Then,
    \begin{equation}
        r(f^*(\rho(a)I))=[r\circ f]^*(\rho(a)I)=\rho(a)[r\circ f]^*(I)=\rho(a)(r(f^*(I)))=r(\rho'(a)f^*(I)).
    \end{equation}
    \begin{equation*}
        \therefore \; f^*(\rho(a)I)=\rho'(a)f^*(I),\; f^* \text{ is a geometry-aware representation w.r.t} \;\gA.
    \end{equation*}
\end{proof}

\cref{lemma:stroke} requires additional conditions: $r$ must be $\gA$-equivariant map and must be injective in the range of $f^*$.
By using the rendering process as a differentiable rasterizer with parameterized Bézier curves described in Appendix \ref{appendix:renderer}, conditions to be $\gA$-equivariance map can be satisfied.

\begin{lemma} \label{lemma:renderer}
    \vspace{0.75em}
    Let $C(s, p)=(1-s)^3t_1 + s(1-s)^2t_2 + s^2(1-s)t_3 + s^3t_4$ a cubic Bézier curve where $p=(t_1, t_2, t_3, t_4, c, w)$, $u\in\Omega$, and $s\in[0,1]$.
    Then a rendering process with a single stroke $[r'(p)](u)=\exp(-\min_s || C(s, p) - u ||^2_2/w^2)\cdot c$ is a $\gA$-equivariant map,
    where $\rho'$ is defined as $\rho'(a)(t_1, t_2, t_3, t_4, c, w)=(a.t_1, a.t_2, a.t_3, a.t_4, c, w)$.
\end{lemma}

\begin{proof}
    \begin{align}
    \begin{split}
        \rho(a)[r'(p)](u) &= [r'(p)](a^{-1}.u) \\
                         &=\exp(-\min_s || C(t, p) - a^{-1}u ||^2_2/w^2)\cdot c \\ 
                         &=\exp(-\min_s || a.[(1-s)^3t_1 + s(1-s)^2t_2 + s^2(1-s)t_3 + s^3t_4] - u ||^2_2/w^2)\cdot c \\ 
                         &=\exp(-\min_s || [(1-s)^3a.t_1 + s(1-s)^2a.t_2 + s^2(1-s)a.t_3 + s^3a.t_4] - u ||^2_2/w^2)\cdot c \\ 
                         &=\exp(-\min_s || C(s, \rho'(a)p) - u ||^2_2/w^2)\cdot c \\ 
                         &=[r'(\rho'(a)p)](u). 
    \end{split}
    \end{align}
    Therefore, the rendering process with a single stroke $r':p\rightarrow\gS$ is a $\gA$-equivariant map.
\end{proof}

Assuming that strokes do not overlap with each other, we can easily expand \cref{lemma:renderer} with the process of a set of strokes $r:\vp\rightarrow\gS$ by composing $r'(p)$ of each stroke into a single image using the composition function described in Appendix \ref{appendix:renderer}.

To make $r:\vp\rightarrow\gS$ injective, we must restrict the range of $f^*$.
For example, the direction of each stroke does not change the rasterized sketch, i.e., $r'(t_1, t_2, t_3, t_4, c, w)=r'(t_4, t_3, t_2, t_1, c, w)$.
Moreover, the order between each stroke is also invariant for sketches if the strokes do not overlap each other, i.e., $r(\{p^{(1)}, p^{(2)}\})=r(\{p^{(2)}, p^{(1)}\})$.
Thus, we treat the strokes generated by Stroke Generator in \Cref{subsec:lbs} as a permutation invariant set, 
forcing the order of each stroke to be left to right with aligning $\vp_{gt}$ as described in Appendix \ref{appendix:implementation}.

Combining \cref{lemma:sketch}, \cref{lemma:stroke}, \cref{lemma:renderer}, we obtain \Cref{prop:proposition}.

\section{Ablation on integrated representation}
\label{appendix:rep_ablation}

\begin{table}[h!]
    \centering
    \caption{Quantitative results on ablating the integrated representation $z_{LBS+}$, using InfoNCE loss as $\gL_{embed}$. CLEVR-C refers to the classification of the colors of the rightmost object on the CLEVR dataset, while CLEVR-S refers to the classification of the shapes.}
    \begin{adjustbox}{width=0.5\linewidth}
    \begin{tabular}{lccccc}
        \toprule
                    & Geoclidean & CLEVR-C & CLEVR-S & STL-10 \\ \midrule
$z_e$               & 24.55\stdv1.05          & 74.35\stdv0.82          & 64.99\stdv0.49          & 77.09\stdv0.22          \\
$z_p$               & 50.60\stdv2.12\grayc          & \textbf{84.70}\stdv0.52\grayc & 63.74\stdv3.02\grayc          & 72.52\stdv0.44\grayc          \\
$z_h$               & 50.68\stdv2.77          & 71.32\stdv0.56          & 52.35\stdv1.36          & 73.29\stdv0.42          \\
$(z_p, z_h)$     & \textbf{50.97}\grayc\stdv1.85 & 78.47\grayc\stdv0.18          & 62.48\stdv2.05\grayc          & 75.01\stdv0.29\grayc          \\
full                & 50.01\stdv1.58          & 83.76\stdv0.54          & \textbf{70.89}\stdv1.56 & \textbf{79.15}\stdv0.22 \\

        \bottomrule
    \end{tabular}\end{adjustbox}
    \label{table:7}
\end{table}

We perform the ablation study on \model~(InfoNCE), which uses InfoNCE loss as $\gL_{embed}$.
As summarized in Table \ref{table:7}, utilizing both $z_p$ and $z_h$ improves performance in various tasks compared to using only $z_e$. 
The interesting point is that the concatenated stroke $z_p$ is competitive with conventional features from the CNN encoder with fully aggregated representations. 
Since $z_h$ seems less important than $z_p$ or $z_e$ for tasks on CLEVR and Geoclidean, we observe that $z_h$ can improve classification accuracy on STL-10 by providing aggregated information about the strokes.

\section{Ablation on model architecture} \label{appendix:architecture}

\begin{table}[h!]
    \centering
    \caption{Quantitative results on ablating CLIP model. The model is trained on the STL-10 dataset and evaluated for tasks on STL-10 (Class) and CLEVR (Color, Size, and Shape).}
    \begin{adjustbox}{width=0.4\linewidth}
        \begin{tabular}{lcccc}
            \toprule
            Architecture & Class & Color & Size & Shape \\ \midrule
            RN101 & 79.15 & \textbf{53.87} & \textbf{92.46} & \textbf{59.52} \\
            RN50 & 79.56\grayc & 47.96\grayc & 89.76\grayc & 55.73\grayc \\
            RN50x4 & \textbf{79.88} & 49.25 & 90.07 & 56.70 \\
            ViT-B/32 & 79.78\grayc & 50.55\grayc & 90.63\grayc & 57.46\grayc\\
            ViT-B/16 & 79.06 & 47.99 & 90.08 & 57.73\\
            \bottomrule
        \end{tabular}\end{adjustbox}
    \label{table:clip_ablation}
\end{table}

Since the authors of CLIP provided pre-trained models with multiple backbone architectures\footnote{https://github.com/openai/CLIP}, we conducted an ablation study on the architecture of the CLIP model used to measure $\gL_{percept}$.
Table~\ref{table:clip_ablation} shows that the performance does not vary significantly across different designs of the CLIP model, with the use of the ResNet101 architecture resulting in better performance on tasks involving the CLEVR dataset.

For the architecture for the image encoder, we observed that training features with the ResNet architecture are substantially better than using ViT~\cite{vit} in terms of both overall sketch quality and quantitative metrics.

\section{Ablation study on loss function}
\label{appendix:loss_ablation}
\begin{figure*}[h!]
    \centering
    \includegraphics[width=0.97\textwidth]{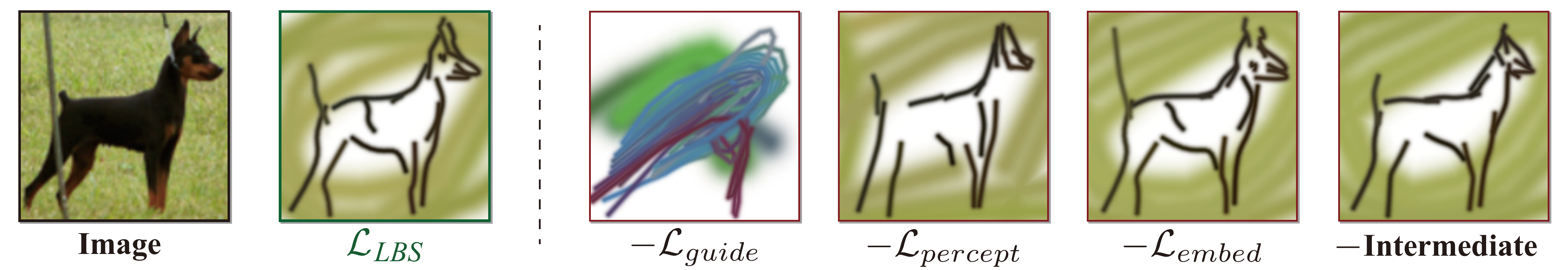}
    \vspace*{-0.5em}
    \caption{
    Qualitative results on ablating loss terms in $\gL_{LBS}$ trained on STL-10 dataset. The importance of $\gL_{guide}$ for training \model{} is highlighted. The visualization shows the qualitative effect of removing each loss component.
    }
    \label{figure:ablation_loss}
    \vspace*{-1em}
\end{figure*}  

\begin{table}[h]
    \centering
    \caption{
    Quantitative results on ablating loss terms in $\gL_{LBS}$ trained on STL-10 dataset.
    The classification result in STL-10 is denoted as Class, while other results are evaluated on the CLEVR dataset.
    \model{} is trained with $\lambda_e=0$, while \model{}~(InfoNCE) is trained with $\lambda_e\neq 0$. The results demonstrate the impact of the proposed loss function on the model's performance.
    }
    \begin{adjustbox}{width=0.5\linewidth}
    \begin{tabular}{lcccccc}
        \toprule
        & Class & RC                & BC                & Size              & Shape          & Material            \\ \midrule
        \model & 78.07             & 53.87             & 61.84             & \textbf{92.46}    & 59.52          & 80.06          \\                   
        \model~(InfoNCE) & \textbf{79.15}\grayc & \textbf{54.53}\grayc    & 62.26\grayc             & 91.81\grayc             & \textbf{59.62}\grayc & \textbf{80.44}\grayc     \\
        $-\gL_{guide}$   & 74.70          & 33.00             & 45.24             & 83.26             & 50.26          & 70.01            \\
        $-\gL_{percept}$  & 78.88\grayc           & 53.79\grayc             & \textbf{62.77}\grayc    & 91.49\grayc             & 59.00\grayc          & 79.96\grayc         \\
        $-$Intermediate  & 78.13             & 50.64             & 58.70             & 91.17             & 57.71          & 77.74          \\
        \bottomrule
        \end{tabular}
        \end{adjustbox}
    \label{table:6}
\end{table}

To investigate the effect of each proposed loss term on the learned features, we conducted experiments by ablating $\gL_{guide}$ and $\gL_{percept}$, and disabling the progressive optimization process for intermediate layers (\ie, using $\gL_{guide}$ with only the final layer's output).
Additionally, we compare the results of \model~(InfoNCE) to examine the impact of $z_h$ on the learned features.
The results presented in \cref{table:6} and \cref{figure:ablation_loss} show that all loss terms contribute to the final performance of the model in terms of generating informative and high-quality sketches. 
By including $\gL_{embed}$, we observed a slight improvement in performance on the CLEVR dataset.
Additionally, we observe a meaningful improvement in high-level inference tasks, such as classification on the STL-10 dataset, while maintaining the quality of the generated sketches.
The loss term $\gL_{guide}$ is crucial for training \model,
and guiding the intermediate layers also improves the quantitative results while generating higher-quality sketches. 
The inclusion of $\gL_{percept}$ in the loss function resulted in a discernible improvement in the quality of the generated sketches with a measurable performance boost in the quantitative comparison.

\newpage

\section{Qualitative results} \label{appendix:qualitative}

\begin{figure*}[h]
    \centering
    \includegraphics[width=0.9\textwidth]{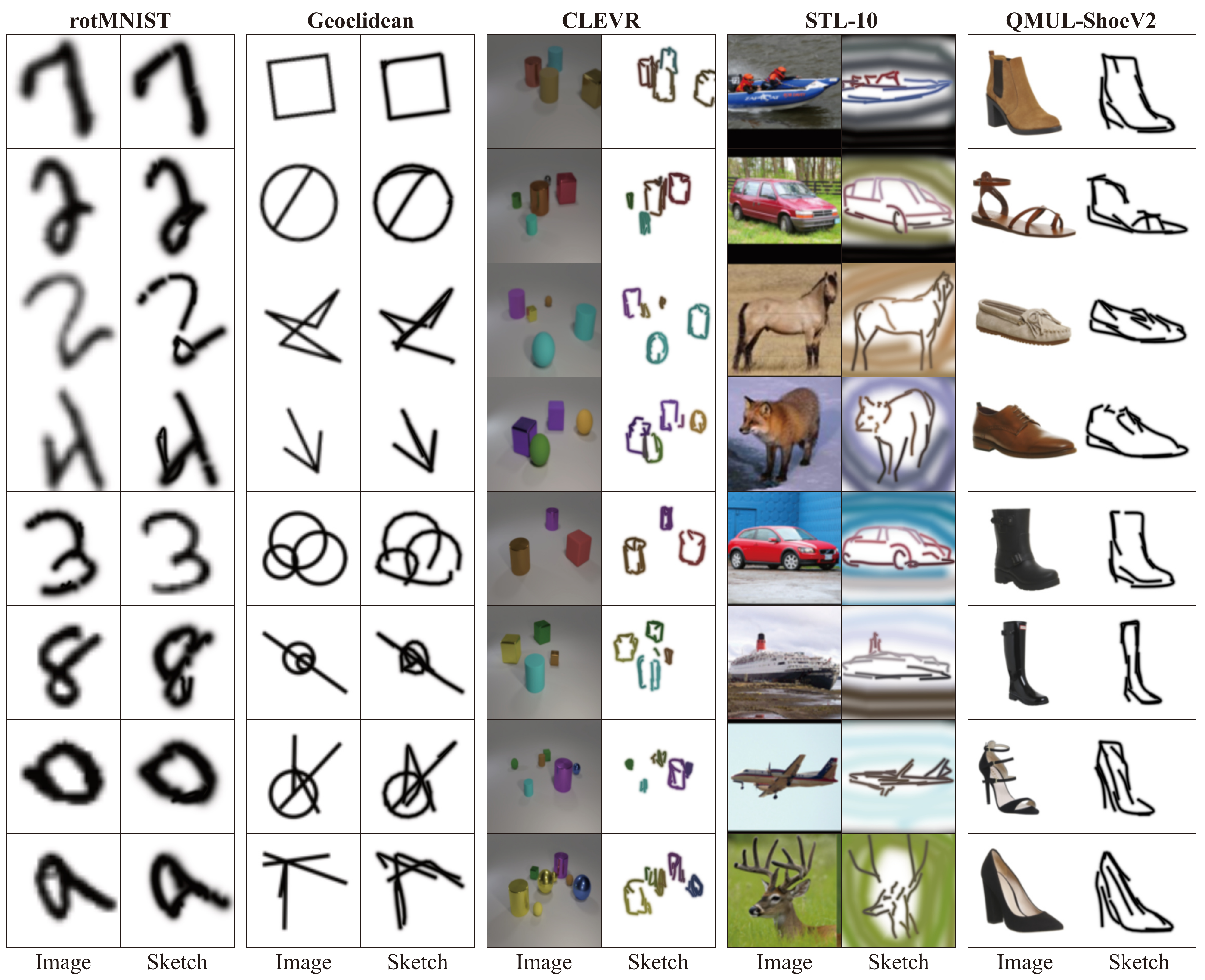}
    \vspace*{-0.5em}
    \caption{Qualitative results on all datasets. We evaluate \model{} against various vision datasets and visualize example sketches for each dataset. \model{} can effectively represent images' geometric information across various datasets, including rotMNIST, Geoclidean, CLEVR, STL-10, and QMUL-ShoeV2.}
    \label{figure:qualitative}
    \vspace*{-1em}
\end{figure*}

\begin{figure*}[h]
    \centering
    \includegraphics[width=0.90\textwidth]{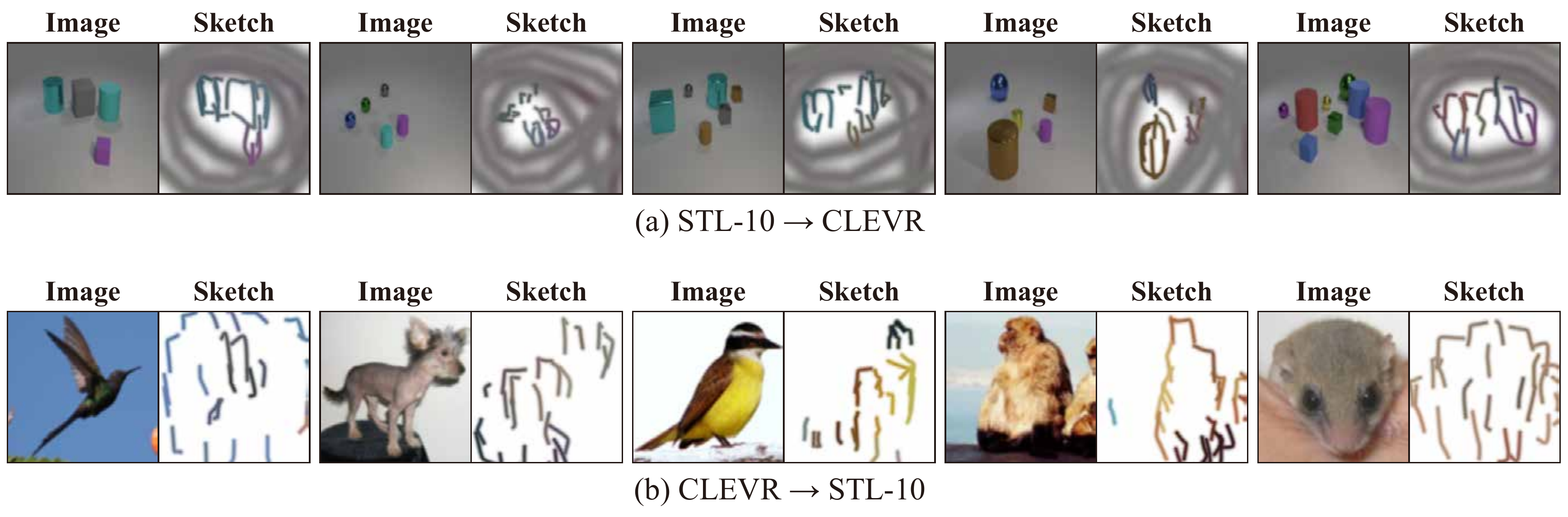}
    \vspace*{-0.5em}
    \caption{Qualitative results of the domain transfer experiment, demonstrating the ability of \model{} to transfer learned geometric information across different domains.
    }	
    \label{figure:Transfer}
    \vspace*{-1em}
\end{figure*}

\clearpage

\begin{figure*}[t]
    \centering
    \includegraphics[width=0.9\textwidth]{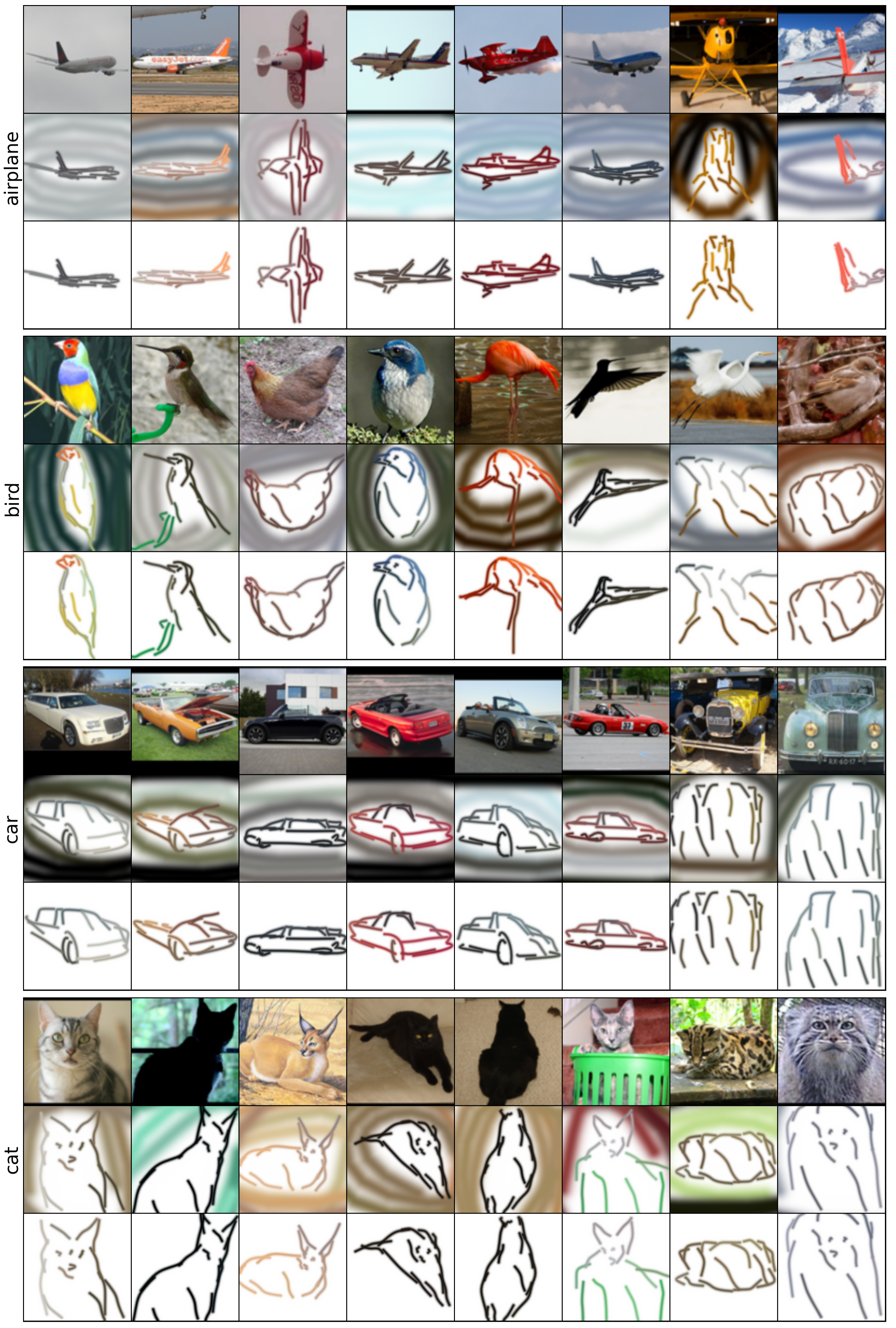}
    \caption{Qualitative results in STL-10. Results for airplane, bird, car, and cat categories on STL-10. Two samples on the right are examples of low-quality generation.}
    \label{figure:qualitative_1}
\end{figure*}

\begin{figure*}[t]
    \centering
    \includegraphics[width=0.9\textwidth]{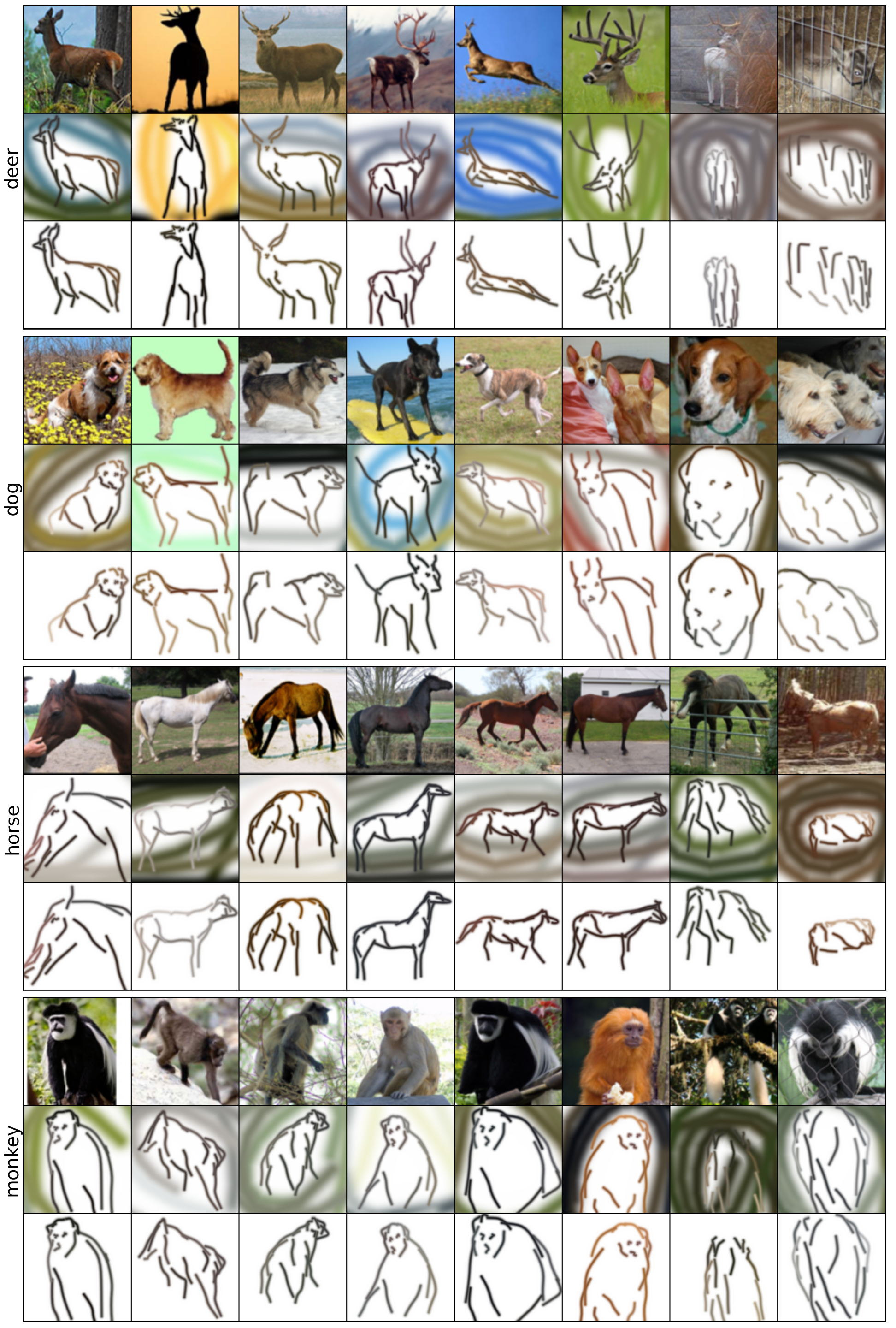}
    \caption{Qualitative results in STL-10. Results for deer, dog, horse, and monkey categories on STL-10. Two samples on the right are examples of low-quality generation.}
    \label{figure:qualitative_2}
\end{figure*}

\begin{figure*}[t]
    \centering
    \includegraphics[width=0.9\textwidth]{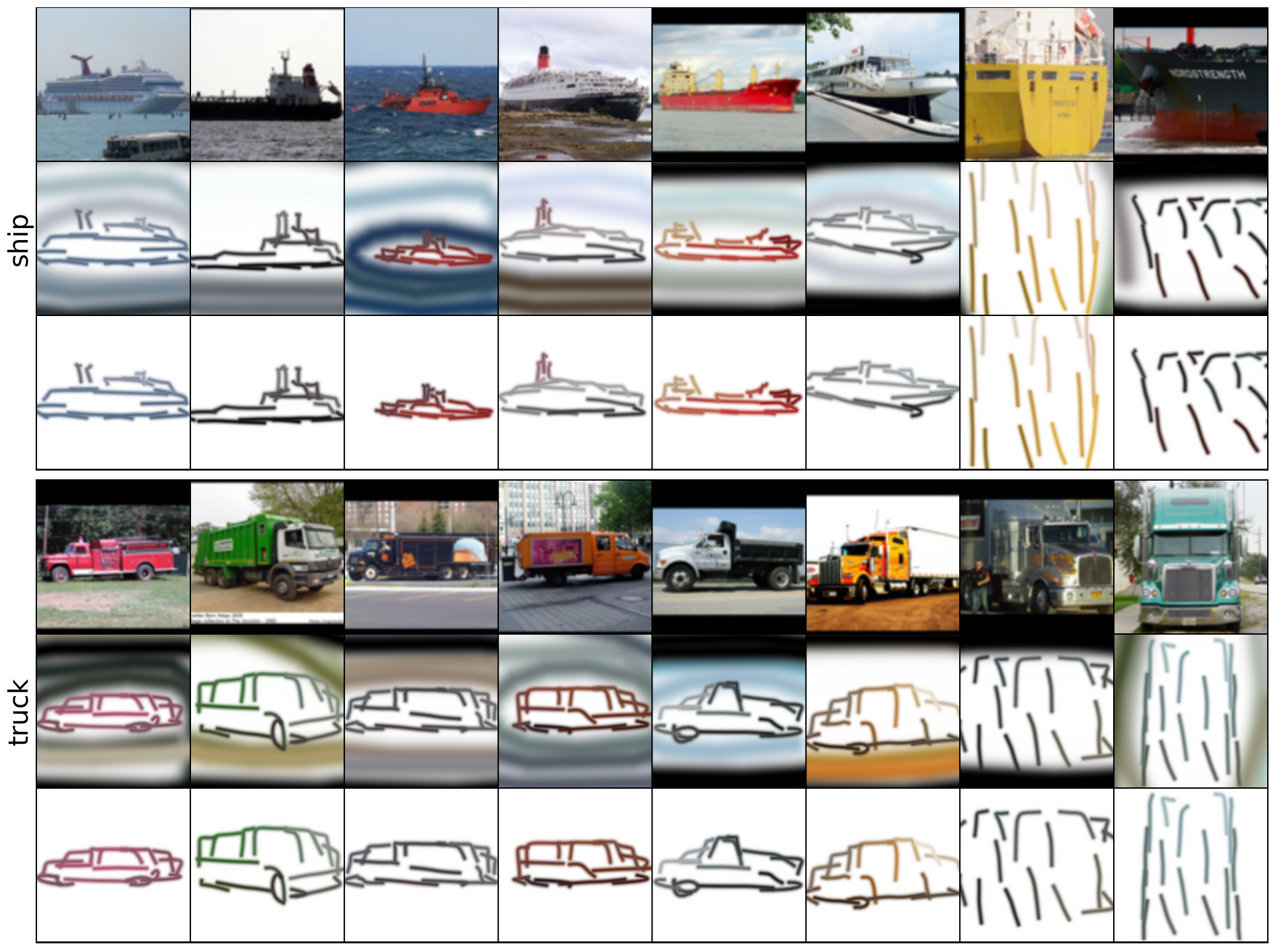}
    \caption{Qualitative results in STL-10. Results for ship and truck categories on STL-10. Two samples on the right are examples of low-quality generation.}
    \label{figure:qualitative_3}
\end{figure*}